\documentclass[journal]{IEEEtran}
\ifCLASSINFOpdf
\else
\fi
\usepackage{verbatim}
\usepackage{graphicx}
\usepackage{subfigure}
\usepackage{cite}
\usepackage{floatrow}

\usepackage[ruled]{algorithm2e}
\usepackage{amssymb,amsmath,amsthm,amsfonts}
\usepackage{flushend}
\usepackage{bm}
\usepackage{amsbsy}
\usepackage{float}
\usepackage{subfigure}
\usepackage{multirow}
\usepackage{tabularx}
\usepackage{multirow}
\usepackage{makecell}
\usepackage{enumerate}
\usepackage{epstopdf}
\makeatletter

\newcommand{\Rmnum}[1]{\expandafter\@slowromancap\romannumeral #1@}
\makeatother
\newtheorem{theorem}{Theorem}
\newtheorem{lemma}{Lemma}

\usepackage{array}
\newfloatcommand{capbtabbox}{table}[][\FBwidth]

\begin{document}
%
\title{Discrete Optimal Graph Clustering}
%
%
%

\author{Yudong Han,
        Lei~Zhu,
        Zhiyong Cheng,
        Jingjing Li,
        Xiaobai Liu
\thanks{This work was
supported National Natural Science Foundation of China under Grant 61802236 and 61806039, in part by the National Postdoctoral Program for Innovative Talents under Grant BX201700045, in part by the China Postdoctoral Science Foundation under Grant 2017M623006.}
\thanks{Y. Han and L. Zhu are with the School of Information Science and Engineering, Shandong Normal University, Jinan 250358, China. (Corresponding author:
Lei Zhu, E-mail: leizhu0608@gmail.com)}
\thanks{Z. Cheng is with the Qilu University of Technology (Shandong Academy of Sciences), Shandong Computer Science Center (National Supercomputer Center in Jinan), Shandong Artificial Intelligence Institute, China.}
\thanks{J. Li is with the School of Computer Science and Engineering, University of Electronic Science and Technology of China, Chengdu 611731, China.}
\thanks{X. Liu is with the Department of Computer Science, San Diego State University, California, 92182, USA.}
}

\markboth{IEEE Transactions on Cybernetics}%
{Shell \MakeLowercase{\textit{et al.}}: Bare Demo of IEEEtran.cls for IEEE Journals}

\maketitle

\begin{abstract}
Graph based clustering is one of the major clustering methods. Most of it work in three separate steps: similarity graph construction, clustering label relaxing and label discretization with k-means. Such common practice has three disadvantages: 1) the predefined similarity graph is often fixed and may not be optimal for the subsequent clustering. 2) the relaxing process of cluster labels may cause significant information loss. 3) label discretization may deviate from the real clustering result since k-means is sensitive to the initialization of cluster centroids. To tackle these problems, in this paper, we propose an effective discrete optimal graph clustering (DOGC) framework. A structured similarity graph that is theoretically optimal for clustering performance is adaptively learned with a guidance of reasonable rank constraint. Besides, to avoid the information loss, we explicitly enforce a discrete transformation on the intermediate continuous label, which derives a tractable optimization problem with discrete solution. Further, to compensate the unreliability of the learned labels and enhance the clustering accuracy, we design an adaptive robust module that learns prediction function for the unseen data based on the learned discrete cluster labels. Finally, an iterative optimization strategy guaranteed with convergence is developed to directly solve the clustering results. Extensive experiments conducted on both real and synthetic datasets demonstrate the superiority of our proposed methods compared with several state-of-the-art clustering approaches.
\end{abstract}


\section{Introduction}
\IEEEPARstart {C}{lustering} is one of the fundamental techniques in machine learning. It has been widely applied to various research fields, such as gene expression\cite{DBLP:journals/tkde/JiangTZ04}, face analysis\cite{DBLP:journals/pami/ElhamifarV13}, image annotation\cite{DBLP:journals/pami/WangZLM08}, and recommendation\cite{clusterrecommend1,clusterrecommend2}. In the past decades, many clustering approaches have been developed, such as k-means\cite{DBLP:journals/J.B.MacQueen}, spectral clustering\cite{DBLP:conf/nips/NgJW01, DBLP:journals/tcyb/CaiC15},\cite{yang2017discrete},\cite{hu2017robust},\cite{yang2015multitask} spectral embedded clustering\cite{DBLP:journals/tnn/NieZTXZ11} and normalized cut\cite{DBLP:conf/cvpr/ShiM97}.

K-means identifies cluster centroids that minimize the within cluster data distances. Due to the simpleness and efficiency, it has been extensively applied as one of the most basic clustering methods. Nevertheless, k-means suffers from the problem of the curse of dimensionality and its performance highly depends on the initialized cluster centroids. As an alternative promising clustering method, spectral clustering and its extensions learn a low-dimensional embedding of the data samples by modelling their affinity correlations with graph \cite{CUI2018,8171745,8425080,LU2019217,8352779,li2018transfer,li2018heterogeneous,li2017low}. These graph based clustering methods,\cite{DBLP:journals/tcyb/LiCD17},\cite{shen2018tpami},\cite{shen2017tmm},\cite{TIP2016binary},
\cite{DBLP:journals/tip/XuSYSL17} generally work in three separate steps: similarity graph construction, clustering label relaxing and label discretization with k-means. Their performance is largely determined by the quality of the pre-constructed similarity graph, where the similarity relations of samples are simply calculated with a fixed distance measurement, which cannot fully capture the inherent local structure of data samples. This unstructured graph may lead to sub-optimal clustering results. Besides, they rely on k-means to generate the final discrete cluster labels, which may result in unstable clustering solution as k-means.

Recently, clustering with adaptive neighbors (CAN)\cite{DBLP:conf/kdd/NieWH14} is proposed to automatically learn a structured similarity graph by considering the clustering performance. Projective clustering with adaptive neighbors (PCAN)\cite{DBLP:conf/kdd/NieWH14} improves its performance further by simultaneously performing subspace discovery, similarity graph learning and clustering. With structured graph learning, CAN and PCAN enhance the performance of graph based clustering further. However, they simply drop the discrete constraint of cluster labels to solve an approximate continuous solution. This strategy may lead to significant information loss and thus reduce the quality of the constructed graph structure. Moreover,  to generate the final discrete cluster labels, graph cut should be exploited in them on the learned similarity graph. To obtain the cluster labels of out-of-sample data, the whole algorithm should be run again. This requirement will bring consideration computation cost in real practice.

In this paper, we propose an effective discrete optimal graph clustering (DOGC) method. We develop a unified learning framework, where the optimal graph structure is adaptively constructed, the discrete cluster labels are directly learned, and the out-of-sample extension can be well supported. In DOGC, a structured graph is adaptively learned from the original data with a guidance of reasonable rank constraint for pursuing the optimal clustering structure. Besides, to avoid the information loss in most graph based clustering methods, a rotation matrix is learned in DOGC to rotate the intermediate continuous labels and directly obtain the discrete ones. Based on the discrete cluster labels, we further integrate a robust prediction module into the DOGC to compensate the unreliability of cluster labels and learn a prediction function for out-of-sample data clustering. To solve the formulated discrete clustering problem, an alternate optimization strategy guaranteed with convergence is developed to iteratively calculate the clustering results. The key advantages of our methods are highlighted as follows:
\begin{table}
\caption{Summary of the main notations.}
\vspace{-3mm}
\label{table:2}
\centering
\begin{tabular}{|p{11mm}<{\centering}|p{67mm}|}
\hline
{Symbols} & {Explanations} \\
\hline
\hline
$d$ & The feature dimension of data \\
\hline
$c$ & The number of clusters\\
\hline
$k$ & The nearest neighbor number of data points\\
\hline
$n$ & Data size\\
\hline
$\alpha,\beta,\gamma$ & Penalty parameters\\
\hline
$\textbf{X}$ & Data matrix \\
\hline
$ \ \textbf{I}_{z}$ & An identity matrix of size $z\times z$ \\
\hline
$\textbf{A}$ & The fixed similarity matrix calculated directly by Gaussian kernel function \\
\hline
$\textbf{S}$ & The learned similarity matrix corresponding to $\textbf{X}$\\
\hline
$\textbf{W}$ & The projection matrix for dimension reduction\\
\hline
$\textbf{P}$ & The mapping matrix from original data to discrete cluster labels\\
\hline
$\textbf{F}$ & The continuous cluster labels \\
\hline
$\textbf{Y}$ & The discrete cluster labels \\
\hline
$\textbf{Q}$ & The rotation matrix \\
\hline
$ \ \ \textbf{D}_{\mathrm{S}}$ & Degree matrix of similarity matrix $\textbf{S}$ \\
\hline
$ \ \ \textbf{L}_{\mathrm{S}}$ & Laplacian matrix of similarity matrix $\textbf{S}$\\
\hline
\end{tabular}
\end{table}
\begin{enumerate}[1.]
\item Rather than exploiting a fixed similarity matrix, a similarity graph is adaptively learned from the raw data by considering the clustering performance. With reasonable rank constraint, the dynamically constructed graph is forced to be well structured and theoretically optimal for clustering.
\item Our model learns a proper rotation matrix to directly generate discrete cluster labels without any relaxing information loss as many existing graph based clustering methods.
\item With the learned discrete cluster labels, our model can accommodate the out-of-sample data well by designing a robust prediction module. The discrete cluster labels of database samples can be directly obtained, and simultaneously the clustering capability for new data can be well supported.
\end{enumerate}

The rest of this paper is organized as follows. Section \Rmnum{2} revisits several representative graph based clustering methods. Section \Rmnum{3} describes the details of the proposed methods. Section \Rmnum{4} introduces the experimental setting. The experimental results are presented in Section \Rmnum{5}. Section \Rmnum{6} concludes the paper.

\section{Graph Clustering Revisited}
\subsection{Notations}
For the data matrix $\textbf{X}=\big\{ \textbf{x}_{1},\ldots,\textbf{x}_{n} \big\} \in \mathbb{R}^{d\times n}$, the $(i,j)_{th}$ entry of $\textbf{X}$ and the $i_{th}$ sample of $\textbf{X}$ are denoted by $x_{ij}$ and $\textbf{x}_{i}$ respectively. The trace of $\textbf{X}$ is denoted by $Tr(\textbf{X})$. The Frobenius norm of matrix $\textbf{X}$ is denoted by $\parallel $\textbf{X}$\parallel_{\mathrm{F}}$. The similarity matrix corresponds to $\textbf{X}$ is denoted by $\textbf{S}$, whose $(i,j)_{th}$ entry is $s_{ij}$. An identity matrix of size $z \times z$ is represented by $\textbf{I}_{z}$ and \textbf{1} denotes a column vector with all elements as 1. The main notations used in this paper are summarized in Table \ref{table:2}.
\subsection{Spectral Clustering}
Spectral clustering\cite{DBLP:conf/nips/NgJW01} requires Laplacian matrix $\textbf{L}_{\mathrm{S}}\in \mathbb{R}^{n\times n}$ as an input. It is computed as $\textbf{L}_{\mathrm{S}}=\textbf{D}_{\mathrm{S}}-\frac{(\textbf{S}^\top +\textbf{S})}{2}$, where $\textbf{D}_{\mathrm{S}}\in \mathbb{R}^{n\times n}$ is a diagonal matrix with the $i_{th}$ diagonal element as $\sum_{j}\frac{(s_{ij}^\top+s_{ij})}{2}$. Supposing there are $c$ clusters in the dataset $\textbf{X}$, spectral clustering solves the following problem:
\begin{equation}
\begin{aligned}
\label{eq:1}
\min_{\textbf{Y}} Tr(\textbf{Y}^\top \textbf{L}_{\mathrm{S}}\textbf{Y}) \quad
 \mathrm{s.t.} \; \textbf{Y}\in \mathrm{Idx}
\end{aligned}
\end{equation}
\noindent where $\textbf{Y}=[\textbf{y}_{1},\textbf{y}_{2},\ldots,\textbf{y}_{n}]^\top \in \mathbb{R}^{n\times c}$ is the clustering indicator matrix and $\textbf{Y}\in \mathrm{Idx}$ means that the clustering label vector of each sample $\textbf{y}_{i}\in \big\{0,1\big\} ^{c\times1}$ contains only one element 1 and the others are 0. As the discrete constraint is imposed on $\textbf{Y}$, Eq.(\ref{eq:1}) becomes a NP-hard problem. To tackle it, most existing methods first relax $\textbf{Y}$ to continuous clustering indicator matrix $\textbf{F}=[\textbf{f}_{1},\textbf{f}_{2},\ldots,\textbf{f}_{n}]^\top \in \mathbb{R}^{n\times c}$, and then calculate the relaxed solution as
\begin{equation}
\begin{aligned}
\label{eq:2}
\min_{\textbf{F}} Tr(\textbf{F}^\top \textbf{L}_{\mathrm{S}}\textbf{F}) \quad
 \mathrm{s.t.} \; \textbf{F}^\top \textbf{F}=\textbf{I}_{c}
\end{aligned}
\end{equation}
\noindent Where the orthogonal constraint $\textbf{F}^\top \textbf{F}=\textbf{I}_{c}$ is adopted to avoid trivial solutions. The optimal solution of $\textbf{F}$ is comprised of $c$ eigenvectors of $\textbf{L}_{\mathrm{S}}$ corresponding to the $c$ smallest eigenvalues. Once $\textbf{F}$ is obtained, k-means is applied to generate the final clustering result. For presentation convenience, we denote $\textbf{F}$ and $\textbf{Y}$ as continuous labels and discrete labels respectively.
\begin{table*}
\caption{Main differences between the proposed methods and representative clustering methods.}
\label{table:1}
\centering
\begin{tabular}{|p{18mm}<{\centering}|p{35mm}<{\centering}|p{20mm}<{\centering}|p{13mm}<{\centering}|p{22mm}<{\centering}|p{24mm}<{\centering}|}
\hline
{Methods} & {Projective} {Subspace} {Learning}& {Information} {Loss} &  {Optimal} {Graph} &{Discrete} {Optimization} &  {Out}-{of}-{Sample} {Extension}\\
\hline
KM & $\times$ & $\times$ & $\times$ & $\times$ & $\times$\\
\hline
N-cut & $\times$ & $\surd$ & $\times$ & $\times$ & $\times$  \\
\hline
R-cut & $\times$ & $\surd$ & $\times$ & $\times$ & $\times$  \\
\hline
CLR  & $\times$ & $\surd$ & $\surd$ & $\times$ & $\times$ \\
\hline
SEC  & $\times$ & $\surd$ & $\times$ & $\times$ & $\times$ \\
\hline
CAN  & $\times$ & $\surd$ & $\surd$  &$\times$ & $\times$ \\
\hline
PCAN & $\surd$ & $\surd$ & $\surd$  & $\times$ & $\times$ \\
\hline
\textbf{DOGC} & $\surd$ & $\times$ & $\surd$ & $\surd$ & $\times$\\
\hline
\textbf{DOGC-OS}& $\surd$ & $\times$ & $\surd$ & $\surd$ & $\surd$\\
\hline
\end{tabular}
\end{table*}
\subsection{Clustering and Projective Clustering with Adaptive Neighbors}
Clustering and projective clustering with adaptive neighbors learns a structured graph for clustering. Given a data matrix $\textbf{X}$, all the data points $\textbf{x} _{j}\mid_{j=1}^{n}$ are connected to $\textbf{x}_{i}$ as neighbors with probability $\textbf{s}_{i}\mid_{i=1}^{n}$. A smaller distance is assigned with a large probability and vice versa. To avoid the case that only the nearest data point is the neighbor of $\textbf{x}_{i}$ with probability 1 and all the other data points are excluded from the neighbor set of $\textbf{x}_i$, a nature solution is to determine the probabilities
$ \textbf{s}_{i}\mid_{i=1}^{n}$ by solving
\begin{equation}
\begin{aligned}
\label{eq:3}
\min_{\textbf{s}_{i}^\top\textbf{1}=1,0\leq s_{ij} \leq 1} \sum_{i,j=1}^{n}\parallel \textbf{x}_{i}-\textbf{x}_{j}\parallel_{\mathrm{F}}^{2}s_{ij}+\xi s_{ij}^{2}
\end{aligned}
\end{equation}
\noindent The second term is a regularization term, and $\xi$ is a regularization parameter.

In the clustering task that partitions the data into $c$ clusters, an ideal neighbor assignment is that the number of connected components is the same as the number of clusters $c$. In most cases, all the data points are connected as just one connected component. In order to achieve an ideal neighbor assignment, the probability $\textbf{s}_{i}\mid_{i=1}^{n}$ is constrained such that the neighbor assignment becomes an adaptive process and the number of connected components is exact $c$. The formula to calculate $\textbf{S}$ in them is:
\begin{equation}
\begin{aligned}
\label{eq:4}
\min_{\textbf{S},\textbf{F}} \sum_{i,j=1}^{n}(\parallel \textbf{x}_{i}-\textbf{x}_{j}\parallel_{\mathrm{F}}^{2}s_{ij}+\xi s_{ij}^{2})+2\lambda Tr(\textbf{F}^\top \textbf{L}_{\mathrm{S}}\textbf{F})\\
\; \mathrm{s.t.} \ \forall i,\textbf{s}_{i}^\top \textbf{1}=1,0 \leq s_{ij} \leq 1,\textbf{F} \in \mathbb{R}^{n\times c},\textbf{F}^\top \textbf{F}=\textbf{I}_{c}
\end{aligned}
\end{equation}
\noindent where $\lambda$ is large enough, $Tr(\textbf{F}^\top \textbf{L}_{\mathrm{S}}\textbf{F})$ is forced to be zero. Thus, the constraint $rank(\textbf{L}_{S})=n-c$ can be satisfied \cite{alavi1991graph}.

\subsection{The Constrained Laplacian Rank for Graph based Clustering}
The constrained Laplacian rank for graph based clustering (CLR)\cite{DBLP:conf/aaai/NieWJH16} follows the same idea of clustering and projective clustering with adaptive neighbors. Differently, it learns a new data matrix $\textbf{S}$ based on the given data matrix $\textbf{A}$ such that $\textbf{S}$ is more suitable for the clustering task. In CLR, the corresponding Laplacian matrix $\textbf{L}_{\mathrm{S}}$ is also constrained as $rank(\textbf{L}_{\mathrm{S}})=n-c$. Under this constraint, all data points can be directly partitioned into exact $c$ clusters\cite{alavi1991graph}. Specifically, CLR solves the following optimization problem
\begin{equation}
\begin{aligned}
\label{eq:5}
& \min_{\textbf{S}}\parallel \textbf{S}-\textbf{A}\parallel_{\mathrm{F}}^{2}\\
& \mathrm{s.t.} \sum_{j}s_{ij}=1, s_{ij} \geq 0, rank(\textbf{L}_{\mathrm{S}})=n-c
\end{aligned}
\end{equation}

\subsection{Key Differences between Our Methods and Existing Works}
Our work is an advocate of discrete optimization of cluster labels, where the optimal graph structure is adaptively constructed, the discrete cluster labels are directly learned, and the out-of-sample extension can be well supported. Existing clustering methods, k-means (KM), normalized-cut (N-cut)\cite{DBLP:journals/tcad/HagenK92}, ratio-cut (R-cut)\cite{DBLP:conf/cvpr/ShiM97}, CLR, spectral embedding clustering (SEC), CAN and PCAN, suffer from different problems. Our methods aim to tackle them in a unified learning framework. The main differences between the proposed methods and existing clustering methods are summarized in Table \ref{table:1}.
\section{The Proposed Methodology}
In this section, we present the details of the proposed methods and introduce an alternative optimization for solving the problems.

\subsection{Overall Formulation}
Most existing graph based clustering methods separate the graph construction and clustering into two independent processes. The unguided graph construction process may lead to sub-optimal clustering result. CAN and PCAN can alleviate the problem. However, they still suffer from the problems of information loss and out-of-sample extension.

In this paper, we propose a unified discrete optimal graph clustering (DOGC) framework to address their problems. DOGC exploits the correlation between similarity graph and discrete cluster labels when performing the clustering. It learns a similarity graph with optimal structure for clustering and directly obtains the discrete cluster labels. Under this circumstance, our model can not only take the advantage of the optimal graph learning, but also obtain discrete clustering results. To achieve above aims, we derive the overall formulation of DOGC as
\begin{equation}
\begin{aligned}
\label{eq:6}
& \min_{\textbf{S},\textbf{F},\textbf{Y},\textbf{Q}} \sum_{i,j=1}^{n}\parallel \textbf{x}_{i}-\textbf{x}_{j}\parallel_{\mathrm{F}}^{2}s_{ij}+\xi s_{ij}^{2}+2\lambda Tr(\textbf{F}^\top \textbf{L}_{\mathrm{S}}\textbf{F})\\
& \quad \quad \quad+\alpha \parallel \textbf{Y}-\textbf{FQ}\parallel_{\mathrm{F}}^{2}\\
& \mathrm{s.t.} \  \textbf{S} \in \mathbb{R}^{n \times n}, \textbf{F} \in \mathbb{R}^{n\times c}, \textbf{F}^\top \textbf{F}=\textbf{I}_{c}, \textbf{Q}^\top \textbf{Q}=\textbf{I}_{c}, \textbf{Y}\in \mathrm{Idx}
\end{aligned}
\end{equation}
\noindent where $\alpha$ and $\xi$ are penalty parameters, $\textbf{Q}$ is a rotation matrix that rotates continuous labels to discrete labels. The $\lambda$ can be determined during the iteration. In each iteration, we can initialize $\lambda=\xi$, then adaptively increase $\lambda$ if the number of connected components of $\textbf{S}$ is smaller than $c$ and decrease $\lambda$ if it is greater than $c$.

In Eq.(\ref{eq:6}), we learn an optimal structured graph and discrete cluster labels simultaneously from the raw data. The first term is to learn the structured graph. To pursue optimal clustering performance, $\textbf{S}$ should theoretically have exact $c$ connected components if there are $c$ clusters. Equivalently, to ensure the quality of the learned graph, the Laplacian matrix $\textbf{L}_{\mathrm{S}}$ should have $c$ zero eigenvalues and the sum of the smallest $c$ eigenvalues, $\sum_{i=1}^{c}\sigma_{i}(\textbf{L}_{\mathrm{S}})$, should be zero. According to Ky Fan theorem\cite{Fan1949On}, $\sum_{i=1}^{c}\sigma_{i}(\textbf{L}_{\mathrm{S}})=\min_{\textbf{F}^\top \textbf{F}=\textbf{I}_{c}} Tr(\textbf{F}^\top \textbf{L}_{\mathrm{S}}\textbf{F})$. Hence, the second term guarantees that the learned $\textbf{S}$ is optimal for subsequent clustering. The third term $\parallel \textbf{Y}-\textbf{FQ}\parallel_{\mathrm{F}}^{2}$ is to find a proper rotation matrix $\textbf{Q}$ that makes $\textbf{FQ}$ close to the discrete cluster labels $\textbf{Y}$. Ideally, if data points $i$ and $j$ belong to different clusters, we should have $s_{ij} = 0$ and vice versa. That is, we have $s_{ij} \neq 0$ if and only if data points $i$ and $j$ are in the same cluster, or equivalently $\textbf{f}_{i}\approx\textbf{f}_{j}$ and $\textbf{y}_{i} = \textbf{y}_{j}$.

The raw features may be high-dimensional and they may contain adverse noises that are detrimental for similarity graph learning. To enhance the robustness of the model, we further extend Eq.(\ref{eq:6}) as
\begin{equation}
\begin{aligned}
\vspace{-4mm}
\label{eq:7}
\min_{\textbf{S},\textbf{F},\textbf{Y},\textbf{Q},\textbf{W}} \underbrace{\sum_{i,j=1}^{n}\frac {\parallel \textbf{W}^\top \textbf{x}_{i}-\textbf{W}^\top \textbf{x}_{j}\parallel_{\mathrm{F}}^{2}s_{ij}}{Tr(\textbf{W}^\top \textbf{XHX}^\top \textbf{W})}+\xi s_{ij}^{2}}_{similarity \quad graph \quad learning}+ \\
\underbrace{ 2\lambda Tr(\textbf{F}^\top \textbf{L}_{\mathrm{S}}\textbf{F})}_{continuous\; label\; learning}+
\underbrace{ \alpha \parallel \textbf{Y}-\textbf{FQ}\parallel_{\mathrm{F}}^{2}}_{discrete\; label\; learning}\\
\mathrm{s.t.}\ \forall i,\textbf{s}_{i}^\top \textbf{1}=1,0 \leq s_{ij} \leq 1,\textbf{F} \in \mathbb{R}^{n\times c},\ \textbf{F}^\top \textbf{F}=\textbf{I}_{c},\\
\textbf{Q}^\top \textbf{Q}=\textbf{I}_{c},\ \textbf{W}^\top \textbf{W}=\textbf{I}_{c},\ \textbf{Y}\in \mathrm{Idx}
\end{aligned}
\end{equation}
\noindent where $\textbf{W}$ is a projection matrix. It maps high-dimensional data into a proper subspace to remove the noises and accelerate the similarity graph learning.

\subsection{Optimization Algorithm for Solving Problem (\ref{eq:7})}
In this subsection, we adopt alternative optimization to solve problem (\ref{eq:7}) iteratively. In particular, we optimize the objective function with respective to one variable while fixing the remaining variables. The key steps are as follows

\textbf{Update $\textbf{\emph{S}}$:} For updating $\textbf{S}$, the problem is reduced to\\
\begin{equation}
\begin{aligned}
\label{eq:8}
& \min_{\textbf{S}} \sum_{i,j=1}^{n}\frac {\parallel \textbf{W}^\top \textbf{x}_{i}-\textbf{W}^\top \textbf{x}_{j}\parallel_{\mathrm{F}}^{2}s_{ij}}{Tr(\textbf{W}^\top \textbf{XHX}^\top \textbf{W})}+\xi s_{ij}^{2}\\
& \quad \quad \quad +2\lambda Tr(\textbf{F}^\top \textbf{L}_{\mathrm{S}}\textbf{F})\\
& \mathrm{s.t.} \ \forall i,\textbf{s}_{i}^\top \textbf{1}=1,0 \leq s_{ij}\leq 1
\end{aligned}
\end{equation}
Since $\sum_{i,j=1}^{n}\parallel \textbf{f}_{i}-\textbf{f}_{j}\parallel_{\mathrm{F}}^{2}s_{ij}=2Tr(\textbf{F}^\top \textbf{L}_{\mathrm{S}}\textbf{F})$, the problem (\ref{eq:8}) can be rewritten as
\begin{equation}
\begin{aligned}
\label{eq:9}
& \min_{\textbf{S}}\sum_{i,j=1}^{n}\frac {\parallel \textbf{W}^\top \textbf{x}_{i}-\textbf{W}^\top \textbf{x}_{j}\parallel_{\mathrm{F}}^{2}s_{ij}}{Tr(\textbf{W}^\top \textbf{XHX}^\top \textbf{W})}+\xi s_{ij}^{2}\\
& \quad \quad \quad +\lambda\parallel \textbf{f}_{i}-\textbf{f}_{j}\parallel_{\mathrm{F}}^{2}s_{ij}\\
& \mathrm{s.t.} \ \forall i,\textbf{s}_{i}^\top \textbf{1}=1,0 \leq s_{ij}\leq 1
\end{aligned}
\end{equation}
In problem (\ref{eq:9}), $\textbf{s}_{i}$ can be solved separately as follows
\begin{equation}
\begin{aligned}
\label{eq:10}
& \min_{\textbf{s}_{i}}\sum_{j=1}^{n}d_{ij}^{wx}s_{ij}+\xi s_{ij}^{2}+\lambda d_{ij}^{f}s_{ij}\\
& \mathrm{s.t.} \ \textbf{s}_{i}^\top \textbf{1}=1,0 \leq s_{ij}\leq 1
\end{aligned}
\end{equation}
\noindent where $d_{ij}^{wx}=\frac{\parallel \textbf{W}^\top \textbf{x}_{i}-\textbf{W}^\top \textbf{x}_{j}\parallel_{\mathrm{F}}^{2}}{Tr(\textbf{W}^\top \textbf{XHX}^\top \textbf{W})}$ and $d_{ij}^{f}=\parallel \textbf{f}_{i}- \textbf{f}_{j}\parallel_{\mathrm{F}}^{2}$.

The optimal solution $\textbf{s}_{i}$ can be obtained by solving the convex quadratic programming problem $\min_{\textbf{s}_{i}^\top\textbf{1}=1,0\leq s_{ij}\leq1}\parallel \textbf{s}_{i}+\frac{1}{2 \xi}\textbf{d}_{i}\parallel_{\mathrm{F}}^{2}$, and $d_{ij}=d_{ij}^{wx}+\lambda d_{ij}^{f}$.

\textbf{Update $\textbf{\emph{F}}$:} For updating $\textbf{F}$, it is equivalent to solve\\
\begin{equation}
\begin{aligned}
\label{eq:11}
\min_{\textbf{F}}Tr(\textbf{F}^\top \textbf{L}_{\mathrm{S}}\textbf{F})+\alpha \parallel \textbf{Y}-\textbf{FQ}\parallel_{\mathrm{F}}^{2} \quad \mathrm{s.t.} \ \textbf{F}^\top \textbf{F}=\textbf{I}_{c}
\end{aligned}
\end{equation}
The above problem can be efficiently solved by the algorithm proposed by \cite{DBLP:journals/mp/WenY13}.

\textbf{Update $\textbf{\emph{W}}$:} For updating $\textbf{W}$, the problem becomes
\begin{equation}
\begin{aligned}
\label{eq:12}
\min_{\textbf{W}^\top \textbf{W}=\textbf{I}_{c}}\sum_{i,j=1}^{n}\frac{\parallel \textbf{W}^\top \textbf{x}_{i}-\textbf{W}^\top \textbf{x}_{j}\parallel_{\mathrm{F}}^{2}s_{ij}}{Tr(\textbf{W}^\top \textbf{XHX}^\top \textbf{W})}
\end{aligned}
\end{equation}
which can be rewritten as
\begin{equation}
\label{eq:12.5}
\begin{aligned}
\min_{\textbf{W}^\top \textbf{W}=\textbf{I}_{c}}\frac{Tr(\textbf{W}^\top \textbf{X} \textbf{L}_{\mathrm{S}} \textbf{X}^\top \textbf{W})}{Tr(\textbf{W}^\top \textbf{XHX}^\top \textbf{W})}
\end{aligned}
\end{equation}
We can solve $\textbf{W}$ using the Lagrangian multiplier method. The Lagrangian function of problem (\ref{eq:12.5}) is
\begin{equation}
\begin{aligned}
\label{eq:13}
\pounds(\textbf{W},\epsilon)=\frac{\parallel \textbf{W}^\top \textbf{x}_{i}-\textbf{W}^\top \textbf{x}_{j}\parallel_{\mathrm{F}}^{2}s_{ij}}{Tr(\textbf{W}^\top \textbf{XHX}^\top \textbf{W})}-\epsilon(Tr(\textbf{W}^\top \textbf{W}-\textbf{I}_{c}))
\end{aligned}
\end{equation}
where $\epsilon$ is the Lagrangian multipliers. Taking derivative $\pounds(\textbf{W},\epsilon)$ w.r.t $\textbf{W}$ and setting it to zero, we have
\begin{equation}
\begin{aligned}
\label{eq:14}
(\textbf{X} \textbf{L}_{\mathrm{S}}\textbf{X}^\top -\frac{\parallel \textbf{W}^\top \textbf{x}_{i}-\textbf{W}^\top \textbf{x}_{j}\parallel_{\mathrm{F}}^{2}s_{ij}}{Tr(\textbf{W}^\top \textbf{XHX}^\top \textbf{W})}\textbf{XHX}^\top )\textbf{W}=\epsilon \textbf{W}
\end{aligned}
\end{equation}
We denote that $\textbf{V}=\textbf{X}\textbf{L}_{S}\textbf{X}^\top -\frac{\parallel \textbf{W}^\top \textbf{x}_{i}-\textbf{W}^\top \textbf{x}_{j}\parallel_{\mathrm{F}}^{2}s_{ij}}{Tr(\textbf{W}^\top \textbf{XHX}^\top \textbf{W})}\textbf{XHX}^\top $. The solution of $\textbf{W}$ in problem (\ref{eq:14}) is formed by $m$ eigenvectors corresponding to the $m$ smallest eigenvalues of the matrix $\textbf{V}$. In optimization, we first fix $\textbf{W}$ in $\textbf{V}$. Then we update $\textbf{W}$ by $\textbf{V} \textbf{W}=\epsilon \textbf{W}$, and assign the obtained $\tilde{\textbf{W}}$ after updating to $\textbf{W}$ in $\textbf{V}$\cite{DBLP:conf/ijcai/WangLNH15}. We iteratively update it until $\mathrm{K.K.T.}$ condition\cite{Lemar2006S} in Eq.(\ref{eq:14}) is satisfied.

\textbf{Update $\textbf{\emph{Q}}$:} For updating $\textbf{Q}$, we have
\begin{equation}
\begin{aligned}
\label{eq:15}
\min_{\textbf{Q}} \parallel \textbf{Y}-\textbf{FQ}\parallel_{\mathrm{F}}^{2} \; \mathrm{s.t.} \ \textbf{Q}^\top \textbf{Q}=\textbf{I}_{c}.
\end{aligned}
\end{equation}
It is the orthogonal Procrustes problem \cite{DBLP:journals/chinaf/Nie0L17}, which admits a closed-form solution.

\textbf{Update $\textbf{\emph{Y}}$:} For updating $\textbf{Y}$, the problem becomes
\begin{equation}
\begin{aligned}
\label{eq:16}
& \min_{\textbf{Y} \in \mathrm{Idx}}\alpha \parallel \textbf{Y}-\textbf{FQ}\parallel_{\mathrm{F}}^{2} \  \mathrm{s.t.} \ \textbf{Y} \in \mathrm{Idx}
\end{aligned}
\end{equation}
Note that $Tr(\textbf{Y}^\top \textbf{Y})=n$, the problem (\ref{eq:16}) can be rewritten as below
\begin{equation}
\begin{aligned}
\label{eq:17}
& \max_{\textbf{Y} \in \mathrm{Idx}}Tr(\textbf{Y}^\top \textbf{PQ}) \ \mathrm{s.t.} \ \textbf{Y} \in \mathrm{Idx}
\end{aligned}
\end{equation}
The optimal solution of $\textbf{Y}$ can be obtained as
\begin{equation}
\begin{aligned}
\label{eq:18}
\textbf{Y}_{ij} = \left\{ \begin{array}{ll}
 1, & \textrm{$j=arg \max_{k} (\textbf{PQ})_{ik}$}\\
 0, & \textrm{otherwise}\\
 \end{array} \right.
\end{aligned}
\end{equation}
The main procedures for solving the problem (\ref{eq:7}) are summarized in Algorithm 1.
\begin{algorithm}[]
            \caption{Optimizing problem (\ref{eq:7})}
            \KwIn{Data matrix $\textbf{X} \in \mathbb{R}^{d \times n}$, cluster number $c$, reduced dimension $m$, parameter $\alpha \geq 0$;}
            \KwOut{$\textbf{F}$, $\textbf{S}$, $\textbf{Y}$, $\textbf{W}$, $\textbf{Q}$;}
            1.Randomly initialize $\textbf{F}$, $\textbf{S}$, $\textbf{Y}$, $\textbf{W}$, $\textbf{Q}$\;
            \While{there is no change on $\textbf{F}$, $\textbf{S}$, $\textbf{Y}$, $\textbf{W}$, $\textbf{Q}$}{
                2.Update $\textbf{L}_{\mathrm{S}}=\textbf{D}_{\mathrm{S}}-\frac{(\textbf{S}^\top +\textbf{S})}{2}$\;
                3.Update $\textbf{S}$ according to the problem (\ref{eq:10})\;
                4.Update $\textbf{F}$ by solving the problem (\ref{eq:11})\;
                5.Update $\textbf{W}$ according to problem (\ref{eq:14})\;
                6.Update $\textbf{Q}$ by solving the problem (\ref{eq:15})\;
                7.Update $\textbf{Y}$ according to the Eq.(\ref{eq:18})\;
           }
           Return $\textbf{F}$, $\textbf{S}$, $\textbf{Y}$, $\textbf{W}$, $\textbf{Q}$\;
    \end{algorithm}
\subsection{Determine the value of $\xi$\cite{DBLP:conf/kdd/NieWH14}}
In practice, regularization parameter is difficult to tune since its value could be from zero to infinite. In this subsection, we present an effective method to determine the regularization parameter $\xi$ in problem (\ref{eq:8}). For each $i$, the objective function in problem (\ref{eq:9}) is equal to the one in problem (\ref{eq:10}). The Lagrangian function of problem (\ref{eq:10}) is
\begin{equation}
\begin{aligned}
\label{eq:19}
\pounds(\textbf{s}_{i},\eta,\bm{\phi})=\frac{1}{2} \parallel \textbf{s}_{i}+\frac{\textbf{d}_{i}^{wx}}{2\xi_{i}} \parallel_{\mathrm{F}}^{2}-\eta(\textbf{s}_{i}^\top \textbf{1}-1)-
\bm{\phi}_{i}^\top \textbf{s}_{i}
\end{aligned}
\end{equation}
where $\eta$ and $\bm{\phi}_{i} \geq \textbf{0}$ are the Lagrangian multipliers.

According to the $\mathrm{K.K.T.}$ condition, it can be verified that the optimal solution $\textbf{s}_{i}$ should be
\begin{equation}
\begin{aligned}
\label{eq:20}
s_{ij}=(-\frac{d_{ij}^{wx}}{2\xi_{i}}+\eta)_{+}
\end{aligned}
\end{equation}
In practice, we could achieve better performance if we focus on the locality of data. Therefore, it is preferred to learn a sparse $\textbf{s}_{i}$, i.e., only the $k$ nearest neighbors of $\textbf{x}_{i}$ have chance to connect to $\textbf{x}_{i}$. Another benefit of learning a sparse similarity matrix $\textbf{S}$ is that the computation burden can be alleviated significantly for subsequent processing.

Without loss of generality, suppose $d_{i1}^{wx}, d_{i2}^{wx}, ..., d_{in}^{wx}$ are ordered from small to large. If the optimal $\textbf{s}_{i}$ has only $k$ nonzero elements, then according to Eq.(\ref{eq:20}), we know $s_{i,k} \geq 0$ and $s_{i,k+1} = 0$. Therefore, we have
\begin{equation}
\begin{aligned}
\label{eq:21}
 -\frac{d_{ik}^{wx}}{2\xi_{i}}+\eta > 0, -\frac{d_{i,k+1}^{wx}}{2\xi_{i}}+\eta \leq 0
\end{aligned}
\end{equation}
According to Eq.(\ref{eq:20}) and the constraint $\textbf{s}_{i}^\top \textbf{1}=1$, we have
\begin{equation}
\begin{aligned}
\label{eq:22}
\sum_{j=1}^{k}(-\frac{d_{ij}^{wx}}{2\xi_{i}}+\eta)=1 \Rightarrow \eta=\frac{1}{k}+\frac{1}{2k\xi_{i}}\sum_{j=1}^{k}d_{ij}^{wx}
\end{aligned}
\end{equation}\\
Hence, we have the following inequality for $\xi$ according to Eq.(\ref{eq:21}) and Eq.(\ref{eq:22}).
\begin{equation}
\begin{aligned}
\label{eq:23}
\frac{k}{2}d_{i,k}^{wx}-\frac{1}{2}\sum_{j=1}^{k}d_{ij}^{wx} < \xi_{i} \leq \frac{k}{2}d_{i,k+1}^{wx}-\frac{1}{2}\sum_{j=1}^{k}d_{ij}^{wx}
\end{aligned}
\end{equation}
Therefore, in order to obtain an optimal solution $\textbf{s}_{i}$ to the problem (\ref{eq:10}) that has exact $k$ nonzero values, we could set
$\xi_{i}$ to be
\begin{equation}
\begin{aligned}
\label{eq:24}
\xi_{i}=\frac{k}{2}d_{i,k+1}^{wx}-\frac{1}{2}\sum_{j=1}^{k}d_{ij}^{wx}
\end{aligned}
\end{equation}
The overall $\xi$ could be set to the mean of $\xi_{1},\xi_{2},\ldots,\xi_{n}$. That is, we could set the $\xi$ to be
\begin{equation}
\begin{aligned}
\label{eq:25}
\xi=\frac{1}{n}\sum_{i=1}^{n}(\frac{k}{2}d_{i,k+1}^{wx}-\frac{1}{2}\sum_{j=1}^{k}d_{ij}^{wx})
\end{aligned}
\end{equation}
The number of neighbors $k$ is much easier to tune than the regularization parameter $\xi$ since $k$ is an integer and it has explicit meaning.
\subsection{Out-of-Sample Extension}
Recall that most existing graph based clustering methods can hardly generalize to the out-of-sample data, which is widely existed in real practice. In this paper, with the learned discrete labels and mapping matrix, we can easily extend DOGC for solving the out-of-sample problem. Specifically, we design an adaptive robust module with $\ell_{2,p}$ loss\cite{DBLP:journals/tmm/YangZGZC14} and integrate them into the above discrete optimal graph clustering model, to learn prediction function for unseen data. In our extended model (DOGC-OS), discrete labels are simultaneously contributed by the original data through the mapping matrix $\textbf{P}$ and the continuous labels $\textbf{F}$ though the rotation matrix $\textbf{Q}$. Specifically, DOGC-OS is formulated as follows
\begin{equation}
\begin{aligned}
\label{eq:26}
& \min_{\textbf{S}, \textbf{F}, \textbf{Y}, \textbf{Q}, \textbf{W}, \textbf{P}} \underbrace{\sum_{i,j=1}^{n}\frac {\parallel
\textbf{W}^\top \textbf{x}_{i}-\textbf{W}^\top \textbf{x}_{j}\parallel_{\mathrm{F}}^{2}s_{ij}}{Tr(\textbf{W}^\top \textbf{XHX}^\top \textbf{W})}+\xi s_{ij}^{2}}_{similarity \quad graph\quad learning}+\\
& \underbrace{ 2\lambda Tr(\textbf{F}^\top \textbf{L}_{\mathrm{S}}\textbf{F})}_{continuous \; label \; learning}+\underbrace{ \alpha \parallel \textbf{Y}-\textbf{FQ}\parallel_{\mathrm{F}}^{2}+\beta \pounds_{2,p}(\textbf{P};\textbf{X}, \textbf{Y})}_{discrete \; label \; learning}
\end{aligned}
\end{equation}
where $\pounds_{2,p}(\textbf{P};\textbf{X}, \textbf{Y})$ is the prediction function learning module. It is calculated as
\begin{equation}
\begin{aligned}
\label{eq:27}
\pounds_{2, p}(\textbf{P};\textbf{X}, \textbf{Y})=\parallel \textbf{Y}-\textbf{X}^\top \textbf{P}\parallel_{2,p}+\gamma \parallel \textbf{P}\parallel_{\mathrm{F}}^{2}
\end{aligned}
\end{equation}
$\textbf{P} \in \mathbb{R}^{d\times c}$ is the projection matrix and the loss function is $\ell_{2,p}(0\leq p\leq2)$ loss, which is capable of alleviating sample noise
\begin{equation}
\begin{aligned}
\label{eq:28}
\parallel \textbf{M}\parallel_{2,p}=\sum_{i=1}^{n}\parallel \textbf{M}_{i}\parallel_{2}^{p}
\end{aligned}
\end{equation}
$\textbf{M}_{i}$ is the $i_{th}$ row of matrix $\textbf{M}$. The above $\ell_{2,p}$ loss not only suppresses the adverse noise but also enhances the flexibility for adapting different noise levels.
\subsection{Optimization Algorithm for Solving Problem (\ref{eq:26})}
Due to the existence of $\ell_{2,p}$ loss, directly optimizing the model turns out to be difficult. Hence, we transform it to an equivalent problem as follows
\begin{equation}
\begin{aligned}
\label{eq:30}
& \min_{\textbf{S}, \textbf{F}, \textbf{Y}, \textbf{Q}, \textbf{W}, \textbf{P}}\sum_{i,j=1}^{n}(\frac {\parallel \textbf{W}^\top \textbf{x}_{i}-\textbf{W}^\top \textbf{x}_{j}\parallel_{\mathrm{F}}^{2}s_{ij}}{Tr(\textbf{W}^\top \textbf{XHX}^\top \textbf{W})}+\xi s_{ij}^{2})\\
& \quad \quad \quad +2\lambda Tr(\textbf{F}^\top \textbf{L}_{\mathrm{S}}\textbf{F})+\alpha \parallel \textbf{Y}-\textbf{FQ}\parallel_{\mathrm{F}}^{2}\\
& \quad \quad \quad +\beta(Tr(\textbf{R}^\top \textbf{DR}) +\gamma \parallel \textbf{P}\parallel_{\mathrm{F}}^{2})\\
& \mathrm{s.t.} \; \forall i,\textbf{s}_{i}^\top \textbf{1}=1,0 \leq s_{ij}\leq 1,\ \textbf{F}^\top \textbf{F}=\textbf{I}_{c},\\ & \textbf{W}^\top \textbf{W}=\textbf{I}_{c}, \textbf{Q}^\top \textbf{Q}=\textbf{I}_{c}, \textbf{Y}\in \mathrm{Idx}
\end{aligned}
\end{equation}
\noindent where $\textbf{D}$ is a diagonal matrix with its $i_{th}$ diagonal element computed as $\textbf{D}_{ii}=\frac{1}{\frac{2}{p}\parallel \textbf{r}_{i}\parallel_{2}^{2-p}}$ and $\textbf{R}=\textbf{Y}-\textbf{X}^\top \textbf{P}$ which is denoted as the loss residual, $\textbf{r}_{i}$ is the $i_{th}$ row of $\textbf{R}$.

The steps of updating $\textbf{S}$, $\textbf{F}$, $\textbf{Q}$, $\textbf{W}$ are similar to that of DOGC except the updating of $\textbf{P}$ and $\textbf{Y}$.

\textbf{Update $\textbf{\emph{P}}$:} For updating $\textbf{P}$, we arrive at
\begin{equation}
\begin{aligned}
\label{eq:32}
\min_{\textbf{P}}Tr((\textbf{Y}-\textbf{X}^\top \textbf{P})\textbf{D}(\textbf{Y}-\textbf{X}^\top \textbf{P}))+\gamma \parallel \textbf{P}\parallel_{\mathrm{F}}^{2}
\end{aligned}
\end{equation}
With the other variables fixed, we arrive at the optimization rule for updating $\textbf{P}$ as
\begin{equation}
\begin{aligned}
\label{eq:33}
\textbf{P}=(\textbf{XDX}^\top +\gamma \textbf{I}_{d})^{-1}\textbf{XDY}
\end{aligned}
\end{equation}

\textbf{Update $\emph{\textbf{Y}}$:} For updating $\textbf{Y}$, we arrive at
\begin{equation}
\begin{aligned}
\label{eq:34}
\min_{\textbf{Y} \in \mathrm{Idx}}\alpha \parallel \textbf{Y}-\textbf{FQ}\parallel_{\mathrm{F}}^{2}+\beta Tr((\textbf{Y}-\textbf{X}^\top \textbf{P})^\top \textbf{D}(\textbf{Y}-\textbf{X}^\top \textbf{P}))
\end{aligned}
\end{equation}
Given the facts that $Tr(\textbf{Y}^\top \textbf{Y})=n$ and $Tr(\textbf{Y}^\top \textbf{DY})=Tr(\textbf{D})$, we can rewrite the above sub-problem as below
\begin{equation}
\begin{aligned}
\label{eq:35}
\max_{\textbf{Y} \in \mathrm{Idx}}Tr(\textbf{Y}^\top \textbf{B})
\end{aligned}
\end{equation}
\noindent where $\textbf{B}=\alpha \textbf{FQ}+\beta \textbf{DX}^\top \textbf{P}$. The above problem can be easily solved as
\begin{equation}
\begin{aligned}
\label{eq:36}
\textbf{Y}_{ij} = \left\{ \begin{array}{ll}
 1, & \textrm{$j=arg \max_{k} \textbf{B}_{ik}$}\\
 0, & \textrm{otherwise}\\
 \end{array} \right.
\end{aligned}
\end{equation}
\subsection{Discussion}
In this subseciton, we discuss the relations of our method DOGC with main graph based clustering methods.
\begin{itemize}
\item \textbf{Connection to Spectral Clustering\cite{DBLP:conf/aaai/HuangNH13a}.}
In our model, $\alpha$ controls the transformation from continuous cluster labels to discrete labels, and $\lambda$ is adaptively updated with the number of connected components in the dynamic graph $\textbf{S}$. When $\textbf{W}$ is a unit matrix, the process of projective subspace learning with $\textbf{W}$ becomes an identity transformation. When $\textbf{S}$ is fixed, it is not a dynamic structure any more and $\lambda$ will remain unchanged. When $\alpha \rightarrow 0$, the effect of the third item in Eq.(\ref{eq:7}) is invalid. Under these circumstances, Eq.(\ref{eq:7}) is equivalent to $\min_{\textbf{F}} Tr(\textbf{F}^\top \textbf{L}_{\mathrm{S}}\textbf{F})$. Thus our model degenerates to the spectral clustering.
\item \textbf{Connection to Optimal Graph Clustering\cite{DBLP:conf/kdd/NieWH14}.}
In DOGC, when $\textbf{W}$ is a unit matrix and $\alpha \rightarrow 0$, the effects of $\textbf{W}$ and $\alpha$ are the same as above. Differently, when $\textbf{S}$ is dynamically constructed, Eq.(\ref{eq:7}) is equivalent to $\min_{\textbf{S},\textbf{F}} \sum_{i,j=1}^{n}(\parallel \textbf{x}_{i}-\textbf{x}_{j}\parallel_{\mathrm{F}}^{2}s_{ij}+\xi s_{ij}^{2})+2\lambda Tr(\textbf{F}^\top \textbf{L}_{\mathrm{S}}\textbf{F})$, where $\textbf{S}$ contains a specific $c$ connected components and $\lambda$ is adjusted by the value of $c$. Under these circumstances, our model degenerates to the optimal graph clustering.
\end{itemize}

\subsection{Complexity Analysis}
As for DOGC, with our optimization strategy, the updating of $\textbf{S}$ requires $O(N^{2})$. Solving $\textbf{Q}$ involves SVD and its complexity is $O(Nc^{2}+c^{3})$. To update $\textbf{F}$, we need $O(Nc^{2}+c^{3})$. To update $\textbf{W}$, two layers of iterations should be performed to achieve convergence. The number of internal iterations is generally a constant, so the time complexity of updating $\textbf{W}$ is $O(N^{2})$. Optimizing $\textbf{Y}$ consumes $O(Nc^{2})$. In DOGC-OS, we need to consider another updating process of $\textbf{D}$ and $\textbf{P}$ which both consume $O(N)$. Hence, the whole time complexity of the proposed methods are all $O(N^{2})$. The computation complexity is comparable to many existing graph-based clustering methods.
\subsection{Convergence Analysis}
In this subsection, we prove that the proposed iterative optimization in Algorithm 1 will converge. Before that, we introduce three lemmas.
\begin{lemma}
\label{lemma:0.5}
For any positive real number $a$ and $b$, we can have the following inequality \cite{DBLP:journals/tip/NieCLL18}:
\begin{equation}
\begin{aligned}
\label{eq:39}
a^{\frac{p}{2}}-\frac{p}{2}\frac{a}{b^{\frac{2-p}{2}}} \leq b^{\frac{p}{2}}-\frac{p}{2}\frac{b}{b^{\frac{2-p}{2}}}
\end{aligned}
\end{equation}
\end{lemma}
\begin{lemma}
\label{lemma:1}
Let $\textbf{r}_{i}$ be the $i_{th}$ row of the residual $\textbf{R}$ in previous iteration, and $\tilde{\textbf{r}_{i}}$ be the $i_{th}$ row of the residual $\tilde{\textbf{R}}$ in current iteration, it has been shown in \cite{DBLP:conf/aaai/KangPCX18} that the following inequality holds:
\begin{equation}
\label{eq:40}
\begin{aligned}
\parallel \tilde{\textbf{r}_{i}} \parallel^{p}-\frac{p\parallel \tilde{\textbf{r}_{i}} \parallel^{2}}{2\parallel \textbf{r}_{i} \parallel^{2-p}} \leq \parallel \textbf{r}_{i} \parallel^{p} -
\frac{p\parallel \textbf{r}_{i} \parallel^{2}}{2\parallel \textbf{r}_{i} \parallel^{2-p}}
\end{aligned}
\end{equation}
\end{lemma}
\begin{lemma}
\label{lemma:2}
Given $\textbf{R}=\big\{ \textbf{r}_{1},\ldots,\textbf{r}_{n} \big\}^\top$, then we have the following conclusion:
\begin{equation}
\label{eq:41}
\begin{aligned}
\sum_{i=1}^{n}\parallel \tilde{\textbf{r}_{i}} \parallel^{p}-\sum_{i=1}^{n}\frac{p\parallel \tilde{\textbf{r}_{i}} \parallel^{2}}{2\parallel \textbf{r}_{i} \parallel^{2-p}} \leq  \sum_{i=1}^{n}\parallel \textbf{r}_{i} \parallel^{p}-\sum_{}\frac{p\parallel \textbf{r}_{i} \parallel^{2}}{2\parallel \textbf{r}_{i} \parallel^{2-p}}
\end{aligned}
\end{equation}
\end{lemma}
\begin{proof}
\label{proof:1}
By summing up the inequalities of all $\textbf{r}_{i}, i=1, 2, \ldots, n$, according to Lemma 2, we can easily reach the conclusion of Lemma 3.
\end{proof}
\begin{theorem}
\label{theorem:1}
In DOGC-OS, updating $\tilde{\textbf{Y}}, \tilde{\textbf{F}}, \tilde{\textbf{Q}}, \tilde{\textbf{W}}, \tilde{\textbf{P}}, \tilde{\textbf{S}}$ will decrease the objective value of problem (\ref{eq:26}) until converge.
\end{theorem}
\begin{proof}
\label{proof:2}
Let $\tilde{\textbf{Y}}, \tilde{\textbf{F}}, \tilde{\textbf{Q}}, \tilde{\textbf{W}}, \tilde{\textbf{P}}, \tilde{\textbf{S}}$ are the optimized solution of the alternative problem (\ref{eq:26}), and we denote
\begin{equation}
\begin{aligned}
\label{eq:42}
\left\{ \begin{array}{ll}
 \psi=\sum_{i,j=1}^{n}\frac {\parallel \textbf{W}^\top \textbf{x}_{i}-\textbf{W}^\top \textbf{x}_{j}\parallel_{\mathrm{F}}^{2}s_{ij}}{Tr(\textbf{W}^\top \textbf{XHX}^\top \textbf{W})},\\
 \tilde{\psi}=\sum_{i,j=1}^{n}\frac {\parallel \tilde{\textbf{W}}^\top \textbf{x}_{i}-\tilde{\textbf{W}}^\top \textbf{x}_{j}\parallel_{\mathrm{F}}^{2}\tilde{s_{ij}}}{Tr(\tilde{\textbf{W}}^\top \textbf{XHX}^\top \tilde{\textbf{W}})}, \\
 \end{array} \right.
\end{aligned}
\end{equation}
It is easy to know that:
\begin{equation}
\begin{aligned}
\label{eq:43}
\frac{p}{2}\frac{\tilde{\psi}^{\frac{2}{p}}}{(\psi^{\frac{2}{p}})^{\frac{2-p}{2}}}+\xi\parallel \tilde{\textbf{S}} \parallel_{\mathrm{F}}^{2} \leq \frac{p}{2}\frac{\psi^{\frac{2}{p}}}{(\psi^{\frac{2}{p}})^{\frac{2-p}{2}}}+\xi\parallel \textbf{S} \parallel_{\mathrm{F}}^{2}
\end{aligned}
\end{equation}
According to Lemma 1, we have
\begin{equation}
\begin{aligned}
\label{eq:44}
(\tilde{\psi}^{\frac{2}{p}})^{\frac{p}{2}}-\frac{p}{2}\frac{\tilde{\psi}^{\frac{2}{p}}}{(\psi^{\frac{2}{p}})^{\frac{2-p}{2}}} \leq
(\psi^{\frac{2}{p}})^{\frac{p}{2}}-\frac{p}{2}\frac{\psi^{\frac{2}{p}}}{(\psi^{\frac{2}{p}})^{\frac{2-p}{2}}}
\end{aligned}
\end{equation}
By summing over Eq.(\ref{eq:43}) and Eq.(\ref{eq:44}) in the two sides, we arrive at
\begin{equation}
\begin{aligned}
\label{eq:45}
\tilde{\psi}+\xi\parallel \tilde{\textbf{S}} \parallel_{\mathrm{F}}^{2} \leq \psi+\xi\parallel \textbf{S} \parallel_{\mathrm{F}}^{2}
\end{aligned}
\end{equation}
We also denote
\begin{equation}
\begin{aligned}
\label{eq:46}
\left\{ \begin{array}{ll}
 \jmath=Tr(\textbf{F}^\top \textbf{L}_{S}\textbf{F})+\alpha \parallel \textbf{Y}-\textbf{FQ} \parallel_{\mathrm{F}}^{2}+\beta\gamma \parallel \textbf{P} \parallel_{\mathrm{F}}^{2}\\
 \tilde{\jmath}=Tr(\tilde{\textbf{F}}^\top \tilde{\textbf{L}_{S}}\tilde{\textbf{F}})+\alpha \parallel \tilde{\textbf{Y}}-\tilde{\textbf{F}}\tilde{\textbf{Q}} \parallel_{\mathrm{F}}^{2}+\beta\gamma \parallel \tilde{\textbf{P}} \parallel_{\mathrm{F}}^{2}\\
 \end{array} \right.
\end{aligned}
\end{equation}
Then, we have
\begin{equation}
\begin{aligned}
\label{eq:47}
\tilde{\jmath}+\beta Tr(\tilde{\textbf{R}}^\top \textbf{D}\tilde{\textbf{R}}) \leq \jmath+\beta Tr(\textbf{R}^\top \textbf{DR})
\end{aligned}
\end{equation}
\begin{equation}
\begin{aligned}
\label{eq:48}
\Rightarrow \tilde{\jmath}+\beta\sum_{i=1}^{n}\frac{p\parallel \tilde{\textbf{r}_{i}}\parallel^{2}}{2\parallel \textbf{r}_{i}\parallel^{2-p}} \leq
\jmath+\beta\sum_{i=1}^{n}\frac{p\parallel \textbf{r}_{i}\parallel^{2}}{2\parallel \textbf{r}_{i}\parallel^{2-p}}
\end{aligned}
\end{equation}

\begin{equation}
\begin{aligned}
\label{eq:50}
\Rightarrow \tilde{\jmath}+\beta\sum_{i=1}^{n}\parallel \tilde{\textbf{r}_{i}} \parallel^{p}-\beta(\sum_{i=1}^{n}\parallel \tilde{\textbf{r}_{i}} \parallel^{p}-\sum_{i=1}^{n}\frac{p\parallel \tilde{\textbf{r}_{i}} \parallel^{2}}{2\parallel \textbf{r}_{i} \parallel^{2-p}}) \leq \\
\jmath+\beta\sum_{i=1}^{n}\parallel \textbf{r}_{i} \parallel^{p}-\beta(\sum_{i=1}^{n}\parallel \textbf{r}_{i} \parallel^{p}-\sum_{i=1}^{n}\frac{p\parallel \textbf{r}_{i} \parallel^{2}}{2\parallel \textbf{r}_{i} \parallel^{2-p}})
\end{aligned}
\end{equation}
With Lemma 3, we have
\begin{equation}
\begin{aligned}
\label{eq:51}
\tilde{\jmath}+\beta\sum_{i=1}^{n}\parallel \tilde{\textbf{r}_{i}} \parallel^{p} \leq \jmath+\beta\sum_{i=1}^{n}\parallel \textbf{r}_{i} \parallel^{p}
\end{aligned}
\end{equation}
By summing over Eq.(\ref{eq:45}) and Eq.(\ref{eq:51}) in the two sides, we arrive at
\begin{equation}
\begin{aligned}
\label{eq:52}
\tilde{\psi}+\xi\parallel \tilde{\textbf{S}} \parallel_{\mathrm{F}}^{2}+\tilde{\jmath}+\beta\sum_{i=1}^{n}\parallel \tilde{\textbf{r}_{i}} \parallel^{p} \leq \\
\psi+\xi\parallel \textbf{S} \parallel_{\mathrm{F}}^{2}+\jmath+\beta\sum_{i=1}^{n}\parallel \textbf{r}_{i} \parallel^{p}
\end{aligned}
\end{equation}
This equation indicates that the monotonic decreasing trend of the objective function in Eq.(\ref{eq:26}) in each iteration.
\end{proof}
\begin{table}
\caption{Discriptions of 12 datasets}
{
\centering
\begin{tabular}{|m{20mm}<{\centering}|m{14mm}<{\centering}|m{16mm}<{\centering}|m{17mm}<{\centering}|}
\hline
{Datasets} & {Sample} & {Feature} & {Class}\\
\hline
Solar&322&12&6\\
\hline
Vehicle&846&18&4\\
\hline
Vote&434&16&2\\
\hline
Ecoli&336&7&8\\
\hline
Wine&178&13&3\\
\hline
Glass&214&9&6\\
\hline
Lenses&24&4&3\\
\hline
Heart&270&13&2\\
\hline
Zoo&101&16&7\\
\hline
Cars&392&8&3\\
\hline
Auto&205&25&6\\
\hline
Balance&625&4&3\\
\hline
\end{tabular}}
\label{table:8}
\end{table}

\section{Experimental configuration}
In this section, we introduce the experimental settings, including experiments datasets, baselines, evaluation metric and implementation details.
\subsection{Experimental datasets}
The experiments are conducted on 12 publicly available datasets, including eight object datasets (i.e., Wine, Ecoli, Vehicle, Auto, Glass, Lenses, Zoo, Cars), one disease dataset (i.e. Heart), one dataset to model psychological experiments (i.e. Balance), one dataset for voting election (i.e. Vote) and one dataset for describing the change about the number of solar flares. All these datasets can be obtained from UCI repository (http://archive.ics.uci.edu/ml/datasets). The descriptions of these 12 datasets are summarized in Table \ref{table:8}.
\subsection{Evaluation Baselines}
In experiments, we compare the proposed DOGC and DOGC-OS with the following clustering methods:
\begin{itemize}
\item \textbf{KM\cite{DBLP:journals/J.B.MacQueen}:} KM learns clustering model by jointly minimizing the distances of similar samples and maximizing that of dissimilar samples.
\item \textbf{R-cut\cite{DBLP:journals/tcad/HagenK92}, N-cut\cite{DBLP:conf/cvpr/ShiM97}:} In this two methods, clusters are represented with subgraphs. R-cut and N-cut simultaneously maximize the weights between the same subgraphs and minimize the weights between different subgraphs.
\item \textbf{NMF\cite{DBLP:conf/icdm/LiD06}:} It first decomposes the nonnegative feature matrix into the product of two nonnegative matrices. Then, k-means is performed on the one of nonnegative matrix with lower matrix dimension to calculate the cluster labels.
\item \textbf{CLR\cite{DBLP:conf/aaai/NieWJH16}:} CLR has two variants: CLR0 and CLR1. The former supports L1-norm regularization term and the latter supports the L2-norm. Instead of using a fixed input similarity matrix, they both first learn the similarity matrix $\textbf{S}$ with exact $c$ connected components based on fixed similarity matrix $\textbf{A}$. Then, graph cut is performed on $\textbf{S}$ to calculate the final cluster labels.
\item \textbf{CAN\cite{DBLP:conf/kdd/NieWH14}:} CAN learns the data similarity matrix by assigning the adaptive neighbors for each data point based on local distances. It imposes the rank constraint on the Laplacian matrix of similarity graph, such that the number of connected components in the resulted similarity matrix is exactly equal to the cluster number.
\item \textbf{PCAN\cite{DBLP:conf/kdd/NieWH14}:} Derived from CAN, PCAN improves its performance further by simultaneously performing subspace discovery, similarity graph learning and clustering.
\end{itemize}

\subsection{Evaluation Metrics}
We employ {Normalized Mutual Information (NMI)}, {Accuracy (ACC)} and {Purity} as main evaluation metrics.
\begin{itemize}
\item \textbf{NMI:} We first define normalized mutual information of two distributions $\tilde{A}$ and $\tilde{B}$ as below:
\begin{equation}
\label{eq:53}
\begin{aligned}
NMI(\tilde{A},\tilde{B})=\frac{H(\tilde{A},\tilde{B})}{\sqrt{\Psi(\tilde{A})\Psi(\tilde{B})}},
\end{aligned}
\end{equation}
\noindent where $H(\tilde{A},\tilde{B})$ computes the mutual information of $\tilde{A}$ and $\tilde{B}$. $\Psi(\cdot)$ is the entropy of a distribution. Denote $n_{i}$ as the number of datums in the $i_{th}$ cluster $\textbf{C}_{i}$ generated by a clustering algorithm, $\hat{n_{j}}$ as the number of data points in the $j_{th}$ ground truth class $\textbf{G}_{j}$, $n_{ij}$ as the number of data occurring in both $\textbf{C}_{i}$ and $\textbf{G}_{j}$. Then, NMI is calculated as follows:
\begin{equation}
\label{eq:54}
\begin{aligned}
NMI=\frac{\sum_{i=1}^{c}\sum_{j=1}^{c}n_{i,j}\log(\frac{n\times n_{i,j}}{n_{i}\hat{n_{j}}})}{\sqrt{(\sum_{i=1}^{c}n_{i}\log\frac{n_{i}}{n})(\sum_{j=1}^{c}\hat{n_{j}}\log\frac{\hat{n_{j}}}{n})}}.
\end{aligned}
\end{equation}
\noindent Larger NMI values indicate better clustering performance.
\item \textbf{ACC:} Denote $\textbf{y}_{i}$ as the resultant cluster label of $\textbf{x}_{i}$ using certain clustering method and $\textbf{g}_{i}$
as the ground truth of $\textbf{x}_{i}$, then we have
\begin{equation}
\label{eq:555}
\begin{aligned}
ACC=\frac{\sum_{i}\delta(\textbf{y}_{i},map(\textbf{g}_{i}))}{n},
\end{aligned}
\end{equation}
\noindent where $\delta(x, y) = 1$ if $x = y$, $\delta(x, y) = 0$ otherwise, and $map(\textbf{g}_{i})$ is the best mapping function that permutes cluster labels to match the ground truth labels. Larger ACC values indicate better clustering performance.
\item \textbf{Purity:} Apart from ACC and NMI, purity is another popularly used evaluation
metric. For ground-truth set $\bm{\mu}=\big\{ \mu_{1},\mu_{2},\ldots,\mu_{n} \big\}$ and clustering result set $\bm{\nu}=\big\{ \nu_{1},\nu_{2},\ldots,\nu_{n} \big\}$, the purity is computed by first assigning each cluster to the class which is the most frequent in the cluster, and then counting the number of correctly assigned objects, finally dividing by $n$:
\begin{equation}
\label{eq:55}
\begin{aligned}
Purity(\bm{\mu},\bm{\nu})=\frac{1}{n}\Sigma_{k}\max_{j}|\nu_{k} \cap \mu_{j}|
\end{aligned}
\end{equation}
Similar to ACC and NMI evaluation metric, the higher the purity, the better clustering performance.
\end{itemize}
\subsection{Implementation Details}
In the experiment, we set the number of clusters to be the ground truth in each dataset. The parameters of all compared algorithms are in arrange of $\big\{10^{-6}, 10^{-4}, 10^{-2}, 1, 10^{2}, 10^{4} \big\}$. For those methods calling for a fixed similarity matrix as an input, like Ratio Cut, Normalized Cut, CLR0, CLR1 and NMF, the graph is constructed with the Gaussian kernel function. As for CAN, PCAN, DOGC and DOGC-OS, we randomly initialize their involved variables. We repeat the clustering process 100 times independently to perform all the methods and record the best result. The best performance of DOGC-OS and DOGC is achieved when $k$ is set to around $\frac{1}{10}$ of the total amount of each dataset. In DOGC, there is only one parameter $\alpha$. When $\alpha$ ranges in $\big\{10^{-6}, 10^{-4}, 10^{-2} \big\}$, we record the best result of DOGC on each dataset. In DOGC-OS, there are three parameters: $\alpha$, $\beta$ and $\gamma$. With $\alpha$ ranging in $\big\{10^{-4}, 10^{-2} \big\}$, we will obtain a generally optimal result on each dataset. We further optimize the results by fixing $\alpha$ and adjusting $\beta$. $\alpha$ is mainly used for discrete label learning, and $\beta$ is a parameter that controls the projection from the raw data to the final cluster labels. The balance of $\alpha$ and $\beta$ is crucial. $\gamma$ is adjusted while the overfitting problem arises. When predicting the new data, we set $\gamma$ to 0.1 or 1. When we pour all the data into model to perform training, we set $\gamma$ to 0.0001.
\begin{table*}
\vspace{-1mm}
\caption{ACC on Real Datasets}
\label{table:9}
\centering
\begin{tabular}{|c|c|c|c|c|c|c|c|c|c|c|}
\hline
\multirow{1}{*}{{Methods}} &\multicolumn{1}{c|}{\textbf{KM}} & \multicolumn{1}{c|}{\textbf{R-Cut}} & \multicolumn{1}{c|}{\textbf{N-Cut}}&\multicolumn{1}{c|}{\textbf{CLR0}}&\multicolumn{1}{c|}{\textbf{CLR1}}&
\multicolumn{1}{c|}{\textbf{NMF}}&\multicolumn{1}{c|}{\textbf{CAN}}&\multicolumn{1}{c|}{\textbf{PCAN}}&
\multicolumn{1}{c|}{\textbf{DOGC}}&\multicolumn{1}{c|}{\textbf{DOGC-OS}}\\
\hline
    Solar & 0.5182 & 0.3498 & 0.3932 & 0.2879 & 0.2879 & 0.5201 & 0.5448 & 0.4396 & 0.4789 & \textbf{0.6149}\\ \hline
    Vehicle & 0.4527 & 0.4598 & 0.4598 & 0.4101 & 0.4101 & 0.4433 & 0.4468 & 0.4527 & 0.4276 & \textbf{0.5532}\\ \hline
    Vote & 0.8345 & 0.5701 & 0.5701 & 0.6206 & 0.6206 & 0.8092 & 0.8667 & 0.9218 & 0.9401 & \textbf{0.9562}\\ \hline
    Ecoli & 0.7679 & 0.5417 & 0.5476 & 0.5565 & 0.5922 & 0.6101 & 0.8053 & 0.8274 & 0.8125 & \textbf{0.8631}\\ \hline
    Wine & 0.7022 & 0.6180 & 0.6180 & 0.5168 & 0.5168 & 0.6685 & 0.9494 & 0.9940 & \textbf{0.9944} & 0.9831\\ \hline
    Glass & 0.5561 & 0.3828 & 0.3826 & 0.4392 & 0.5093 & 0.3785 & 0.5000 & 0.4953 & \textbf{0.6075} & 0.5981\\ \hline
    Lenses & 0.6250 & 0.5000 & 0.5000 & 0.5000 & 0.4583 & 0.6667 & 0.7624 & \textbf{0.8750} & \textbf{0.8750} & \textbf{0.8750}\\ \hline
    Heart & 0.5926 & 0.6259 & 0.6296 & 0.6222 & 0.6111 & 0.6296 & 0.5963 & 0.7148 & \textbf{0.8556} & 0.8370\\ \hline
    Zoo & 0.8416 & 0.5149 & 0.5149 & 0.4455 & 0.4455 & 0.8020 & 0.7921 & 0.8218 & \textbf{0.8916} & 0.8812\\ \hline
    Cars & 0.4490 & 0.6301 & 0.6378 & 0.6173 & 0.6173 & 0.6122 & 0.6250 & 0.5791 & 0.6039 & \textbf{0.6735}\\ \hline
    Auto & 0.3659 & 0.3220 & 0.3171 & 0.3610 & 0.3610 & 0.3415 & 0.3463 & 0.3268 & 0.4146 & \textbf{0.4488}\\ \hline
    Balance & 0.6400 & 0.5296 & 0.5264 & 0.6368 & 0.6624 & 0.6656 & 0.5616 & 0.5936 & 0.6823 & \textbf{0.7312}\\ \hline
\end{tabular}
\vspace{-5mm}
\end{table*}
\begin{table*}
\caption{NMI on Real Datasets}
\label{table:10}
\centering
\begin{tabular}{|c|c|c|c|c|c|c|c|c|c|c|}
\hline
\multirow{1}{*}{{Methods}} & \multicolumn{1}{c|}{\textbf{KM}} & \multicolumn{1}{c|}{\textbf{R-Cut}} & \multicolumn{1}{c|}{\textbf{N-Cut}}&\multicolumn{1}{c|}{\textbf{CLR0}}&\multicolumn{1}{c|}{\textbf{CLR1}}&
\multicolumn{1}{c|}{\textbf{NMF}}&\multicolumn{1}{c|}{\textbf{CAN}}&\multicolumn{1}{c|}{\textbf{PCAN}}&
\multicolumn{1}{c|}{\textbf{DOGC}}&\multicolumn{1}{c|}{\textbf{DOGC-OS}}\\
\hline
    Solar & 0.4131 & 0.1884 & 0.2119 & 0.2618 & 0.2618 & 0.3375 & 0.3869 & 0.2549 & 0.2876 & \textbf{0.4219}\\ \hline
    Vehicle & 0.1800 & 0.1881 & 0.1928 & 0.1574 & 0.1574 & 0.1369 & 0.2070 & 0.0530 & 0.1980 & \textbf{0.2370}\\ \hline
    Vote & 0.3658 & 0.0745 & 0.0745 & 0.1275 & 0.1275 & 0.3067 & 0.4320 & 0.5888 & 0.6616 & \textbf{0.6852}\\ \hline
    Ecoli & 0.5606 & 0.5207 & 0.5210 & 0.5173 & 0.4836 & 0.5045 & 0.7220 & 0.7244 & 0.6425 & \textbf{0.6901}\\ \hline
    Wine & 0.8385 & 0.8562 & 0.8792 & 0.3144 & 0.3144 & 0.8324 & 0.8897 & 0.9425 & \textbf{0.9729} & 0.9261\\ \hline
    Glass & 0.3575 & 0.3215 & 0.2858 & 0.3266 & 0.3266 & 0.2870 & 0.2691 & 0.3382 & \textbf{0.3601} & 0.3575\\ \hline
    Lenses & 0.4696 & 0.2197 & 0.1619 & 0.1619 & 0.1396 & 0.3097 & 0.3977 & 0.6227 & \textbf{0.6652} & \textbf{0.6652}\\ \hline
    Heart & 0.0190 & 0.0437 & 0.0482 & 0.0350 & 0.0357 & 0.0494 & 0.1227 & 0.1330 & \textbf{0.4042} & 0.3556\\ \hline
    Zoo & 0.7803 & 0.5926 & 0.6119 & 0.3770 & 0.3770 & 0.7483 & 0.7446 & 0.7366 & 0.8168 & \textbf{0.8259}\\ \hline
    Cars & 0.1910 & 0.1948 & 0.2020 & 0.2025 & 0.2025 & 0.1713 & 0.2747 & 0.2686 & 0.2724 & \textbf{0.2926}\\ \hline
    Auto & 0.1090 & 0.1346 & 0.1331 & 0.1690 & 0.1691 & 0.0285 & 0.0426 & 0.0703 & \textbf{0.2257} & 0.2074\\ \hline
    Balance & 0.2966 & 0.1469 & 0.1432 & 0.0922 & 0.1147 & 0.2259 & 0.1510 & 0.1221 & 0.2831 & \textbf{0.3093}\\ \hline
\end{tabular}
\vspace{-2mm}
\end{table*}
\section{Experimental results}
In this section, we evaluate the performance of the proposed methods on both synthetic and real datasets. First, we compare our method with the baselines on 12 real datasets. Then, we demonstrate the effects of the proposed methods on discrete label learning, optimal graph learning, projective subspace learning, and out-of-sample extension. Next, parameter experiment is carried out to evaluate the robustness of the proposed methods. Finally, the convergence of the proposed methods is verified by the experimental results.
\begin{figure*}
\centering
\subfigure[KM]{\includegraphics[width=38.5mm]{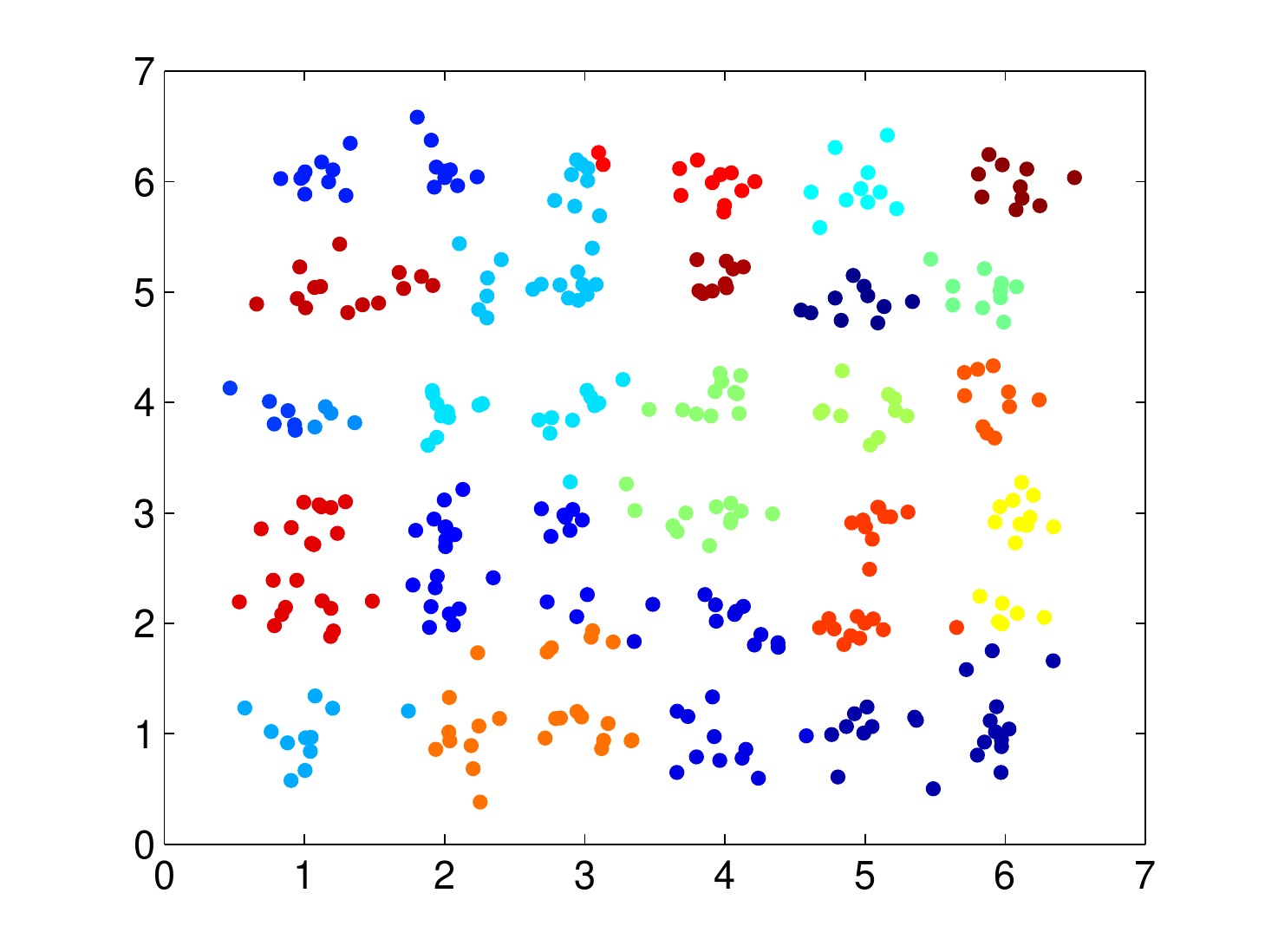}}
\hspace{-0.2in}
\subfigure[DOGC-\uppercase\expandafter{\romannumeral1}]{\includegraphics[width=38.5mm]{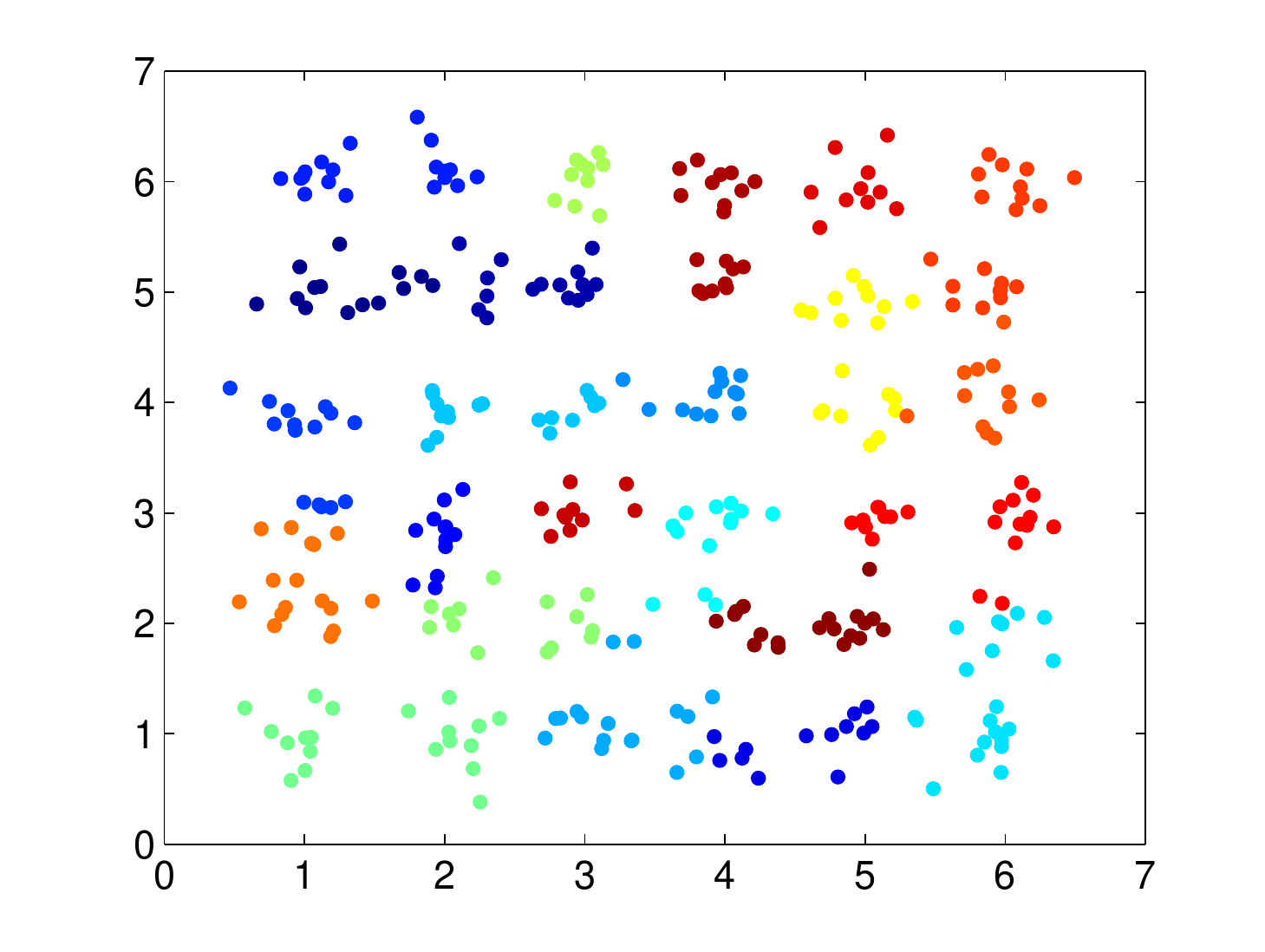}}
\hspace{-0.2in}
\subfigure[DOGC-OS]{\includegraphics[width=38.5mm]{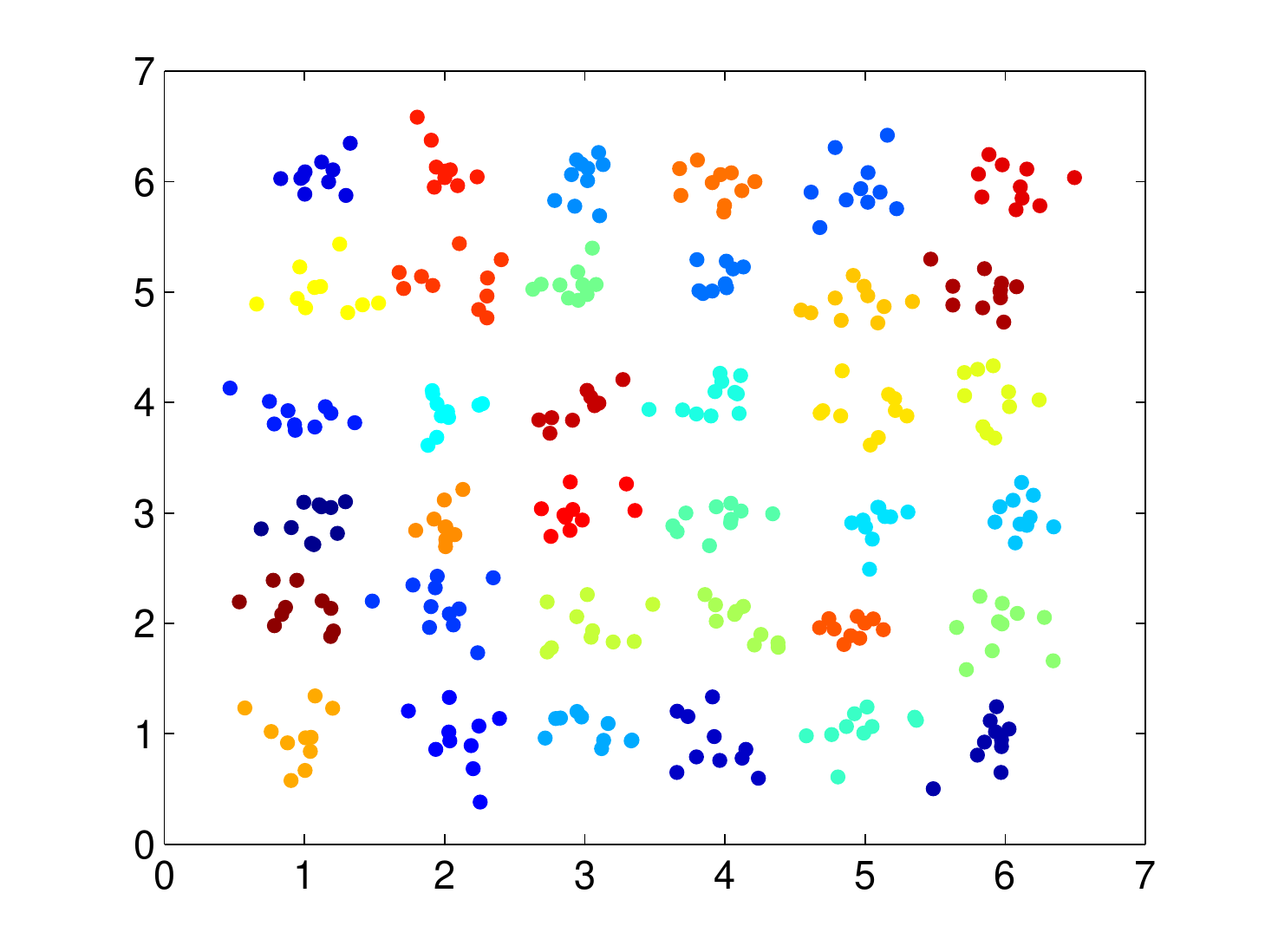}}
\hspace{-0.2in}
\subfigure[]{\includegraphics[width=38.5mm]{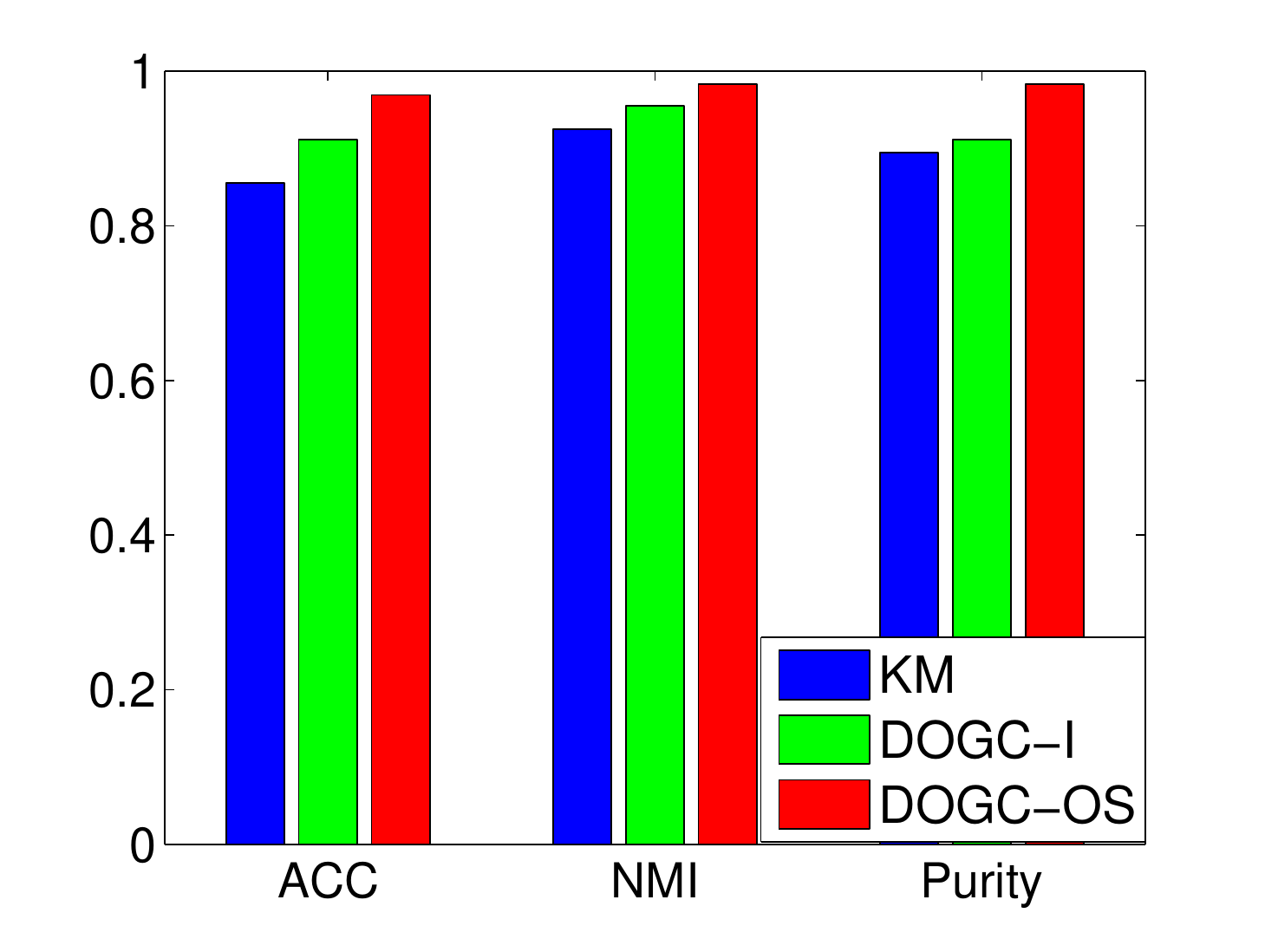}}
\hspace{-0.2in}
\subfigure[Noise sensitivity (ACC)]{\includegraphics[width=38.5mm]{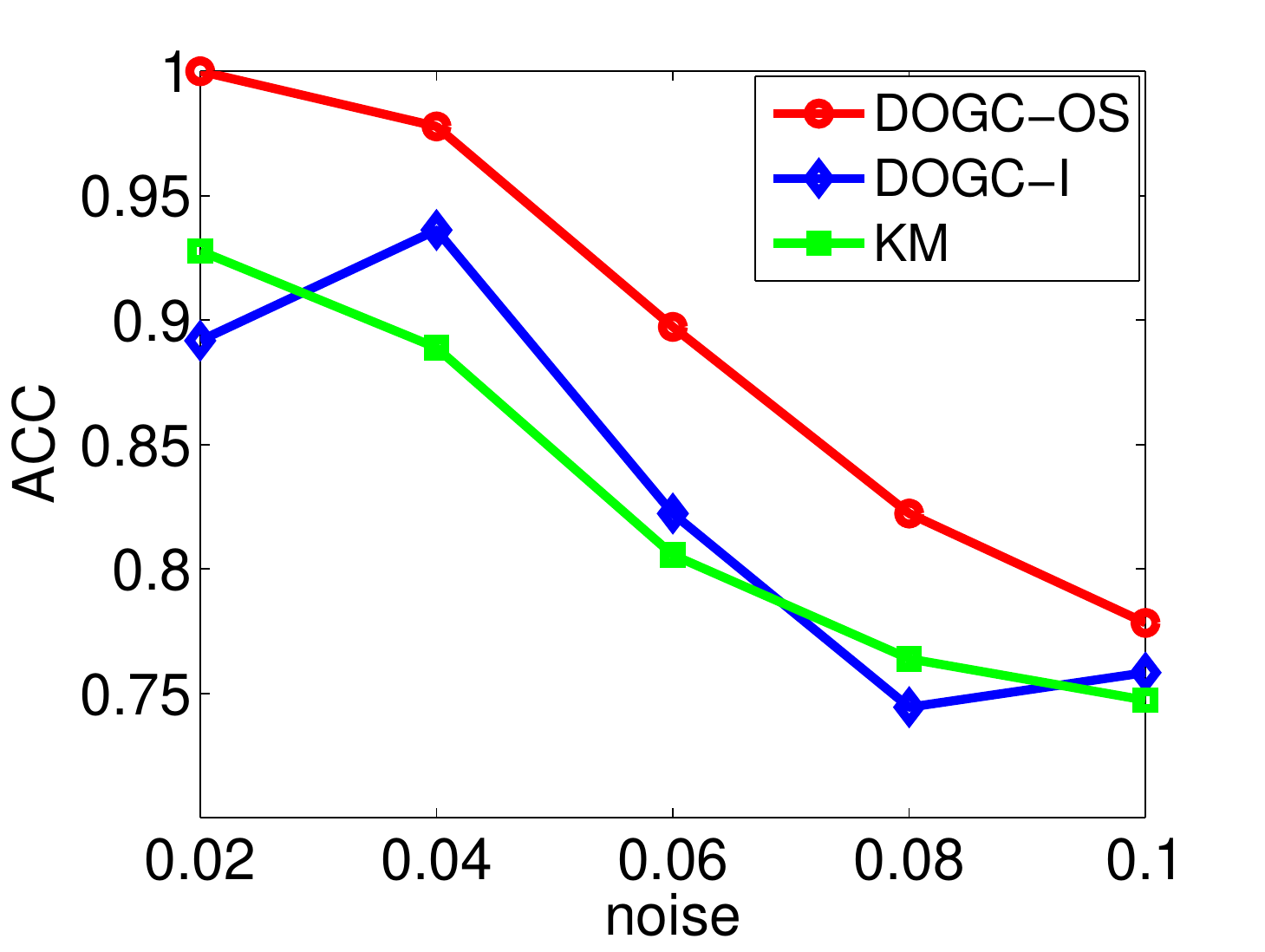}}
\caption{Clustering results on the 36 multi-clusters synthetic data.}
\label{fig:1}
\vspace{-4mm}
\end{figure*}
\begin{figure*}
\centering
\subfigure[KM]{\includegraphics[width=45mm]{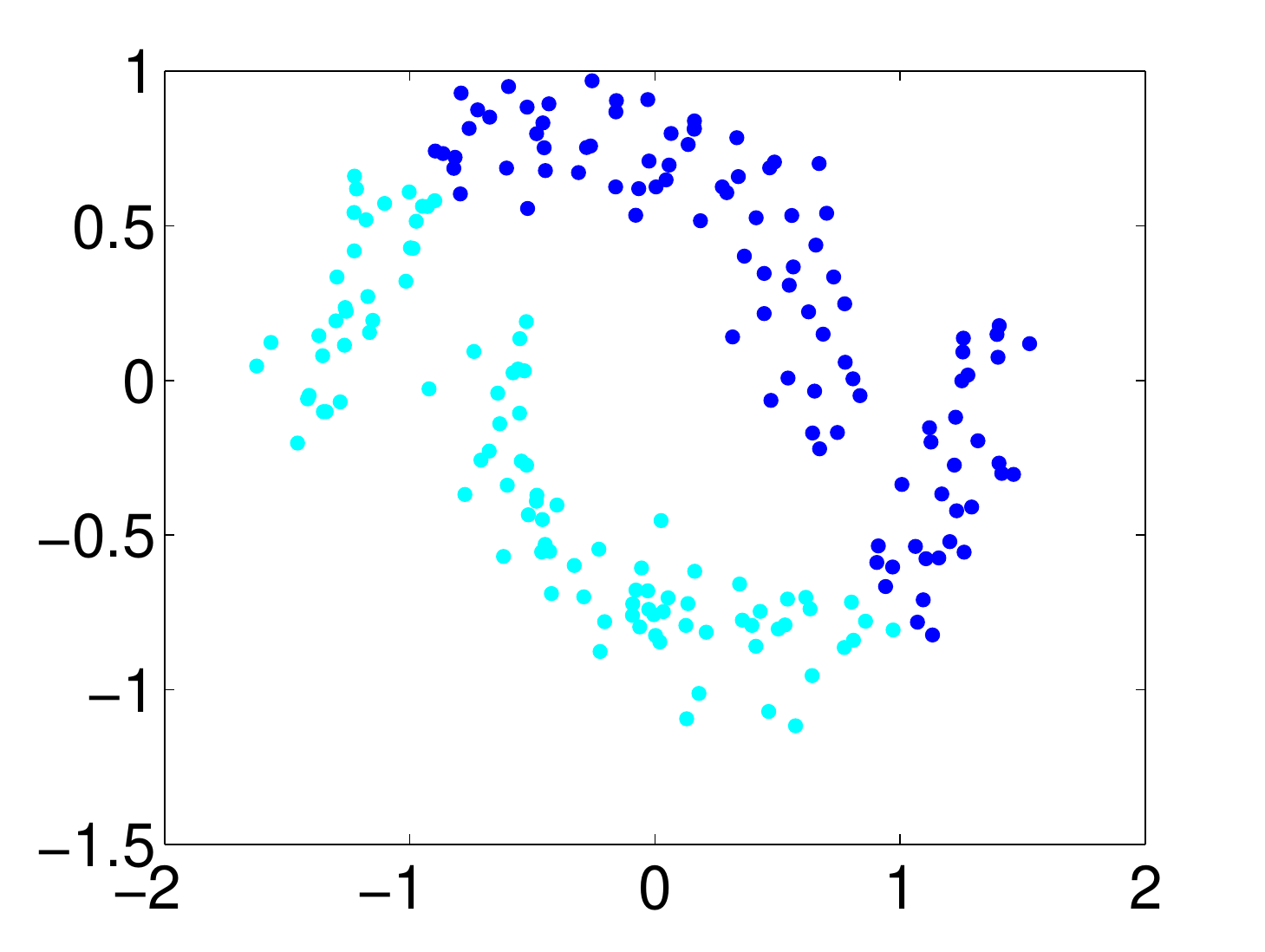}}
\hspace{-0.21in}
\subfigure[DOGC-\uppercase\expandafter{\romannumeral2}]{\includegraphics[width=45mm]{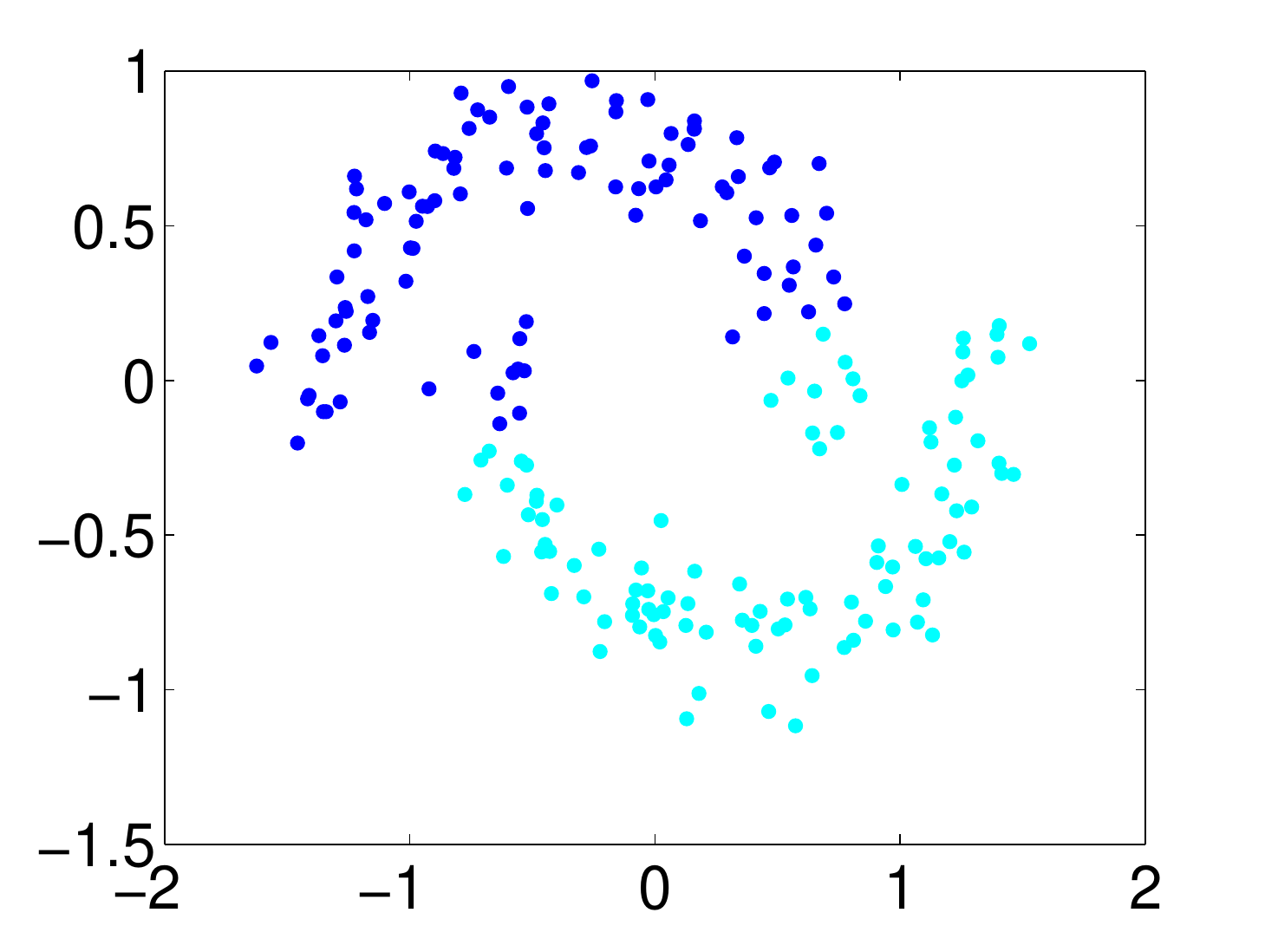}}
\hspace{-0.21in}
\subfigure[DOGC-OS]{\includegraphics[width=45mm]{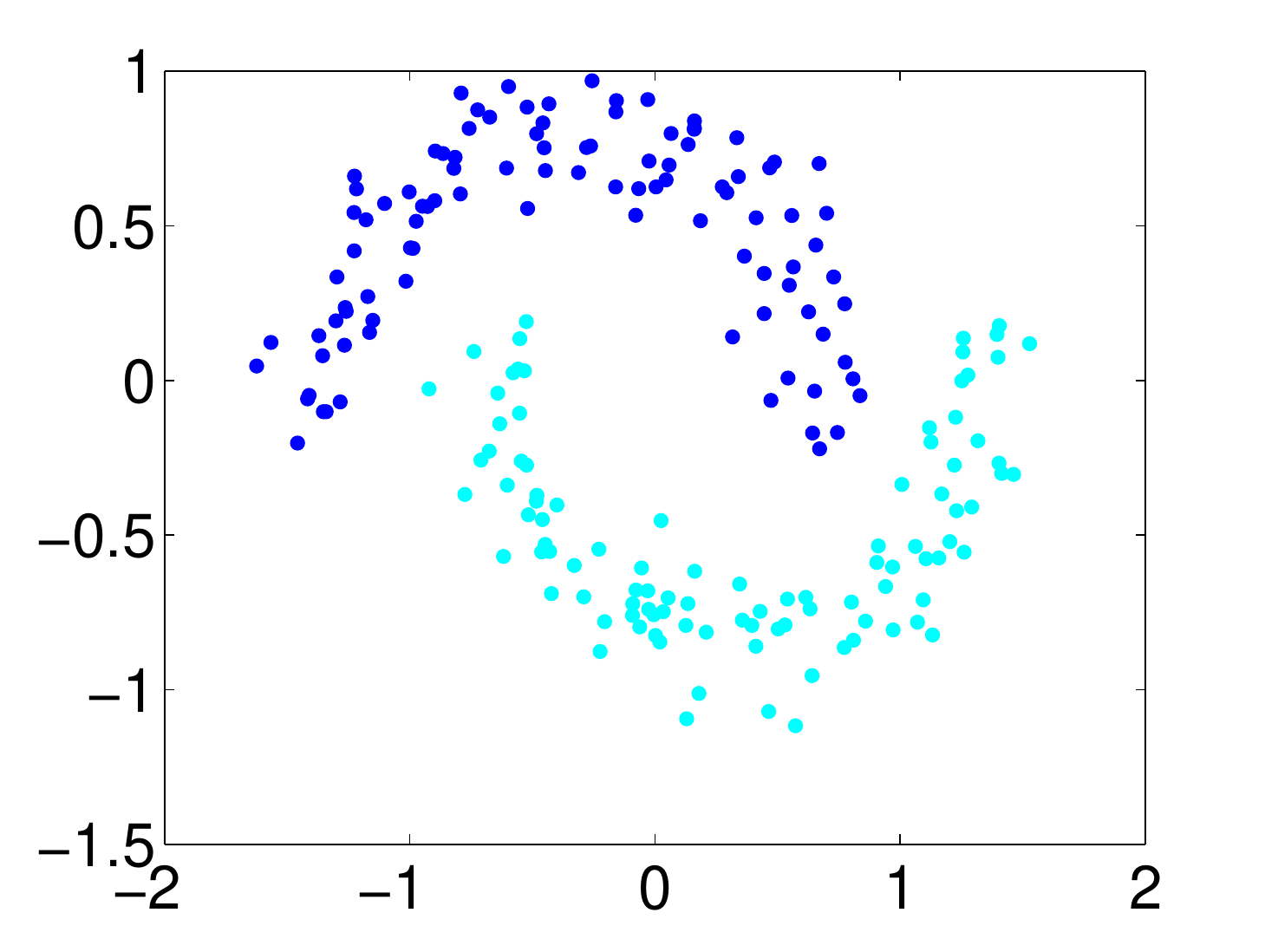}}
\hspace{-0.21in}
\subfigure[]{\includegraphics[width=45mm]{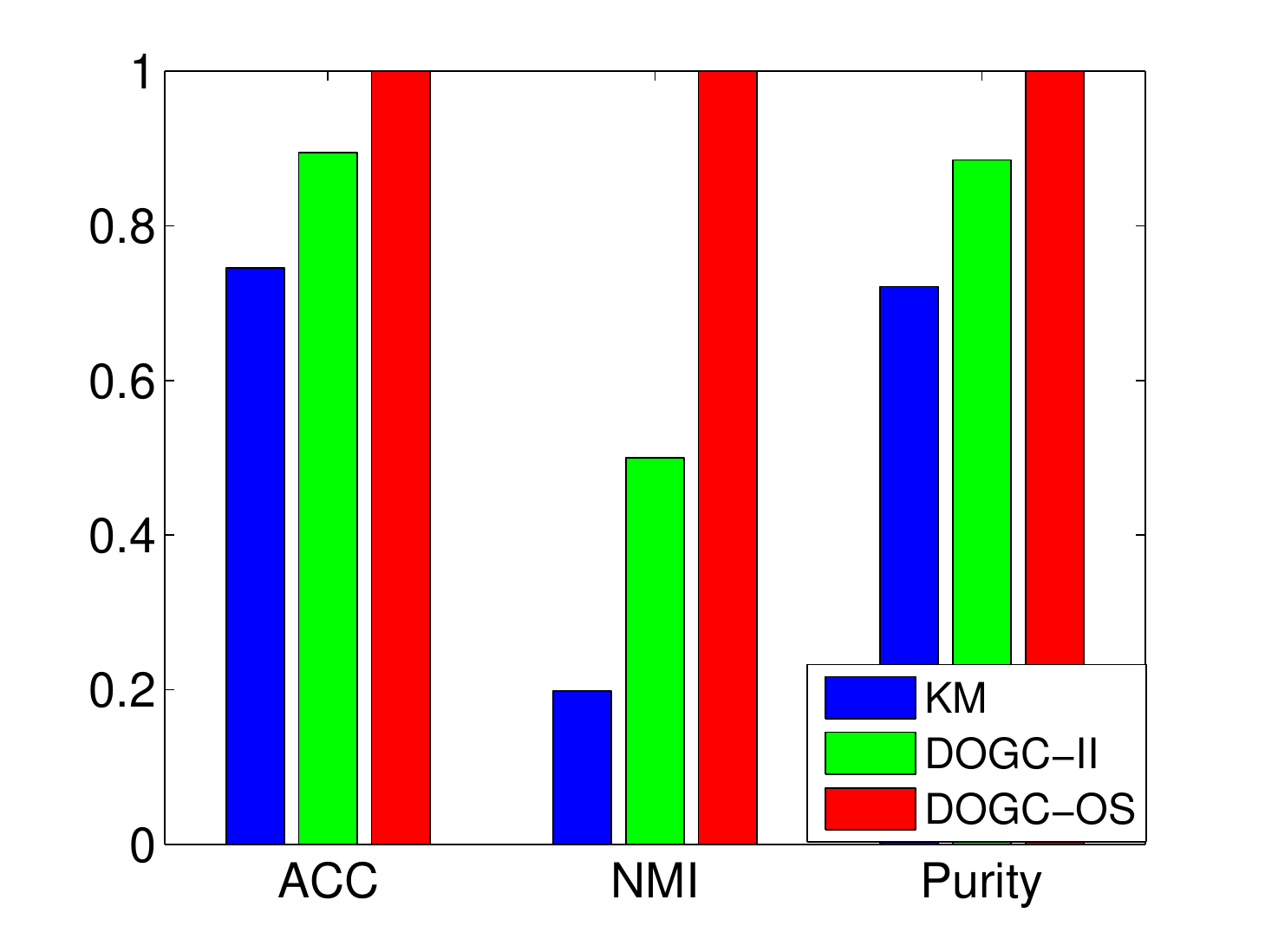}}
\caption{Clustering results on the two-moon synthetic data.}
\label{fig:2}
\end{figure*}

\subsection{Performance Comparison}
Table \ref{table:9} and Table \ref{table:10} present the main ACC and NMI comparison. The presented results clearly demonstrate that DOGC and DOGC-OS consistently outperform the compared approaches on all real datasets. On Wine, the accuracy of DOGC can nearly reach 1. On most datasets, our methods outperform the second best baseline by more than 0.02. In particular, DOGC-OS achieves an amazing improvement of 0.1408 on Heart compared to the second best baseline PCAN. Compare with the best clustering method PCAN on Ecoli, the proposed DOGC-OS obtains an absolute improvement of 0.0357. In compared approaches, graph based clustering methods without optimal graph learning generally achieve worse performance. This may attribute to their fixed similarity graph which is not optimal for the subsequent clustering. In addition, it is interesting to find that optimal graph clustering methods may not obtain better performance than the graph based approaches in certain case. This may be because that the insufficient input samples in these datasets cannot provide enough information for learning a well structured graph for clustering. Under this circumstance, the performance of CAN and PCAN may be impaired. Finally, we analyze why
DOGC-OS outperforms DOGC on some datsets, in the DOGC-OS model, the final clustering indicator matrix $\textbf{Y}$ stems from two transformations $\parallel \textbf{Y}-\textbf{FQ} \parallel_{\mathrm{F}}^{2}$ and $\parallel \textbf{Y} - \textbf{X}^\top \textbf{P} \parallel_{\mathrm{F}}^{2}$. Only $\parallel \textbf{Y}-\textbf{FQ} \parallel_{\mathrm{F}}^{2}$ is in DOGC. It can be seen that DOGC-OS whose $\textbf{Y}$ is under two transformations guidance should be better than DOGC that is guided from $\parallel \textbf{Y}-\textbf{FQ} \parallel_{\mathrm{F}}^{2}$ only.

\subsection{Effects of Discrete Label Learning}
Our methods can directly solve discrete cluster labels without any relaxing. To evaluate the effects of discrete label learning, we compare the performance of DOGC-OS with a variant of our method DOGC-I that relaxes the discrete labels to continuous labels on 36 muti-clusters synthetic dataset. This synthetic dataset is a randomly generated multi-cluster data, there are 36 clusters distributed in a spherical way. Figure \ref{fig:1} and Table \ref{table:6} show the experimental results. From them, we can clearly observe that DOGC-OS fully separates the data (as shown in Figure \ref{fig:1}) and achieves superior clustering performance than DOGC-\uppercase\expandafter{\romannumeral1} on 5 UCI real datasets (as shown in Table \ref{table:6}). Further, we set noise level of 36 muti-clusters synthetic dataset in the range from 0.02 to 0.1 with the interval of 0.01 and observe the performance. We run k-means, DOGC-\uppercase\expandafter{\romannumeral1} and DOGC-OS 100 times and report the best result. Figure \ref{fig:1} reports the results. From it, we can find that DOGC-OS consistently achieves higher clustering accuracy than that of DOGC-\uppercase\expandafter{\romannumeral1} and k-means under different noise levels.
\begin{table*}
\begin{floatrow}
\capbtabbox{
 \hspace{-4mm}
 \begin{tabular}{|c|c|c|c|c|c|c|c|}
 \hline
\multirow{2}{*}{{Methods}} & \multicolumn{3}{c|}{{ACC}} & \multicolumn{3}{c|}{{NMI}}\\
\cline{2-7}
& DOGC-\uppercase\expandafter{\romannumeral1}  &DOGC-\uppercase\expandafter{\romannumeral2} & DOGC-OS &DOGC-\uppercase\expandafter{\romannumeral1}  &DOGC-\uppercase\expandafter{\romannumeral2} & DOGC-OS  \\
\hline
Wine & 0.9752 & 0.9752 & \textbf{0.9831} & 0.8750 & 0.8989 & \textbf{0.9261} \\
\hline
Solar & 0.5627 & 0.6111 & \textbf{0.6149} & 0.3998& 0.4169 & \textbf{0.4219}\\
\hline
Vehicle & 0.4785 & 0.5532 & \textbf{0.5532} & 0.1980& 0.2261 & \textbf{0.2370}\\
\hline
Vote & 0.9382 & 0.9494 & \textbf{0.9562} & 0.7364 & 0.6784 & \textbf{0.6852}\\
\hline
Ecoli& 0.7981 & 0.8251 & \textbf{0.8631} & 0.5742 & 0.5898 & \textbf{0.6901}\\
\hline
 \end{tabular}}
{
 \caption{Performance of DOGC-\uppercase\expandafter{\romannumeral1}, DOGC-\uppercase\expandafter{\romannumeral2} and DOGC-OS on five UCI datasets. The best result is marked with bold.}
 \label{table:6}
}
\end{floatrow}
\end{table*}
\subsection{Effects of Optimal Graph Learning}
In this subsection, we conduct experiment to investigate the effects of optimal graph learning in our methods. To this end, we compare the performance of DOGC-OS with a variant of our methods DOGC-\uppercase\expandafter{\romannumeral2} that removes optimal graph learning function on the two-moon synthetic dataset. In experiments, DOGC-\uppercase\expandafter{\romannumeral2} exploits a fixed similarity graph for input. The two-moon data is randomly generated and there are two data clusters distributed in two-moon shape. Our goal is to divide the data points into exact two clusters.
Figure \ref{fig:2} and Table \ref{table:6} show the experimental results. From them, we can clearly observe that our methods can clearly partition the two-moon data (as shown in Figure \ref{fig:2}) and achieve superior clustering performance than DOGC-\uppercase\expandafter{\romannumeral2} on 5 UCI real datasets (as shown in Table \ref{table:6}). These results demonstrate that the optimal graph learning in our methods can indeed discover the intrinsic data structure and thus improve the clustering methods.
\subsection{Effects of Projective Subspace Learning}
Both DOGC and DOGC-OS can discover a discriminative subspace for data clustering. To validate the effects of projective subspace learning, we compare our methods with PCA\cite{DBLP:conf/icisc/SouissiNGDF10} and LPP\cite{DBLP:conf/nips/HeN03} on two-Gaussian data\cite{DBLP:conf/nips/Zelnik-ManorP04}. In this dataset, two clusters of data are randomly generated to obey the Gaussian distribution. In experiments, we observe their separation capability by varying the distance of two clusters. Figure \ref{fig:3} shows the main results. From it, we can observe that all these four methods can easily find a proper projection direction when two clusters are far from each other. However, as the distance between these two clusters reduces down, PCA becomes ineffective. As the two clusters become closer, LPP fails to achieve the projection goal. However, both DOGC and DOGC-OS always perform well. Theoretically, PCA only focuses on the global structure. Thus it will fail immediately when two clusters become closer. LPP pays more attention on preserving the local structure. It could still achieve satisfactory performance when two clusters are relatively close. Nevertheless, when the distance of two clusters becomes fairly small, LPP is also incapable any more. Different from them, DOGC and DOGC-OS can always keep a satisfactory separation capability consistently as they could identify a discriminative projective subspace with the force of reasonable rank constraint.
\begin{figure*}
\centering
\subfigure[Clusters are far away]{\includegraphics[width=60mm]{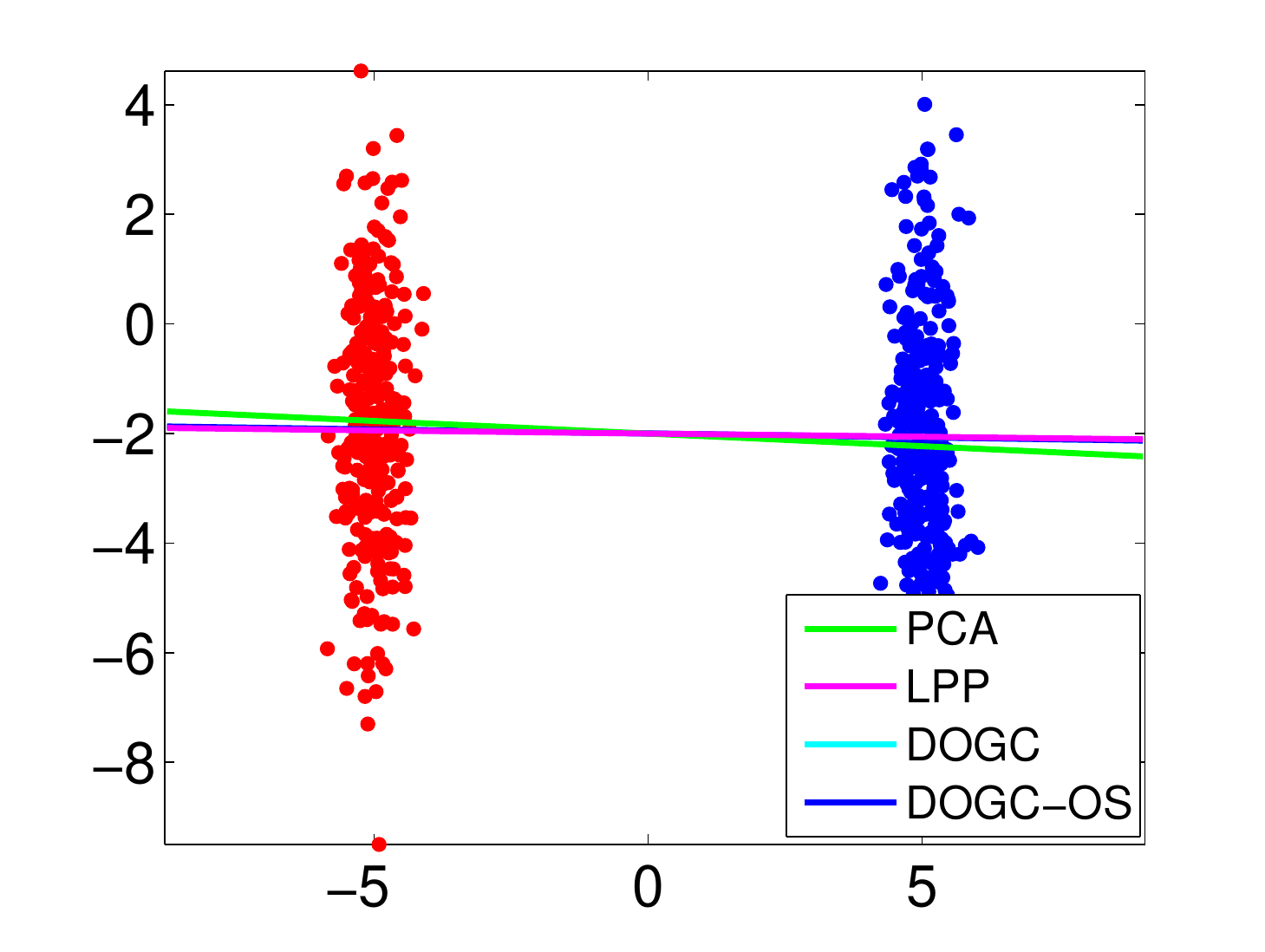}}
\hspace{-0.3in}
\subfigure[Clusters are relatively close]{\includegraphics[width=60mm]{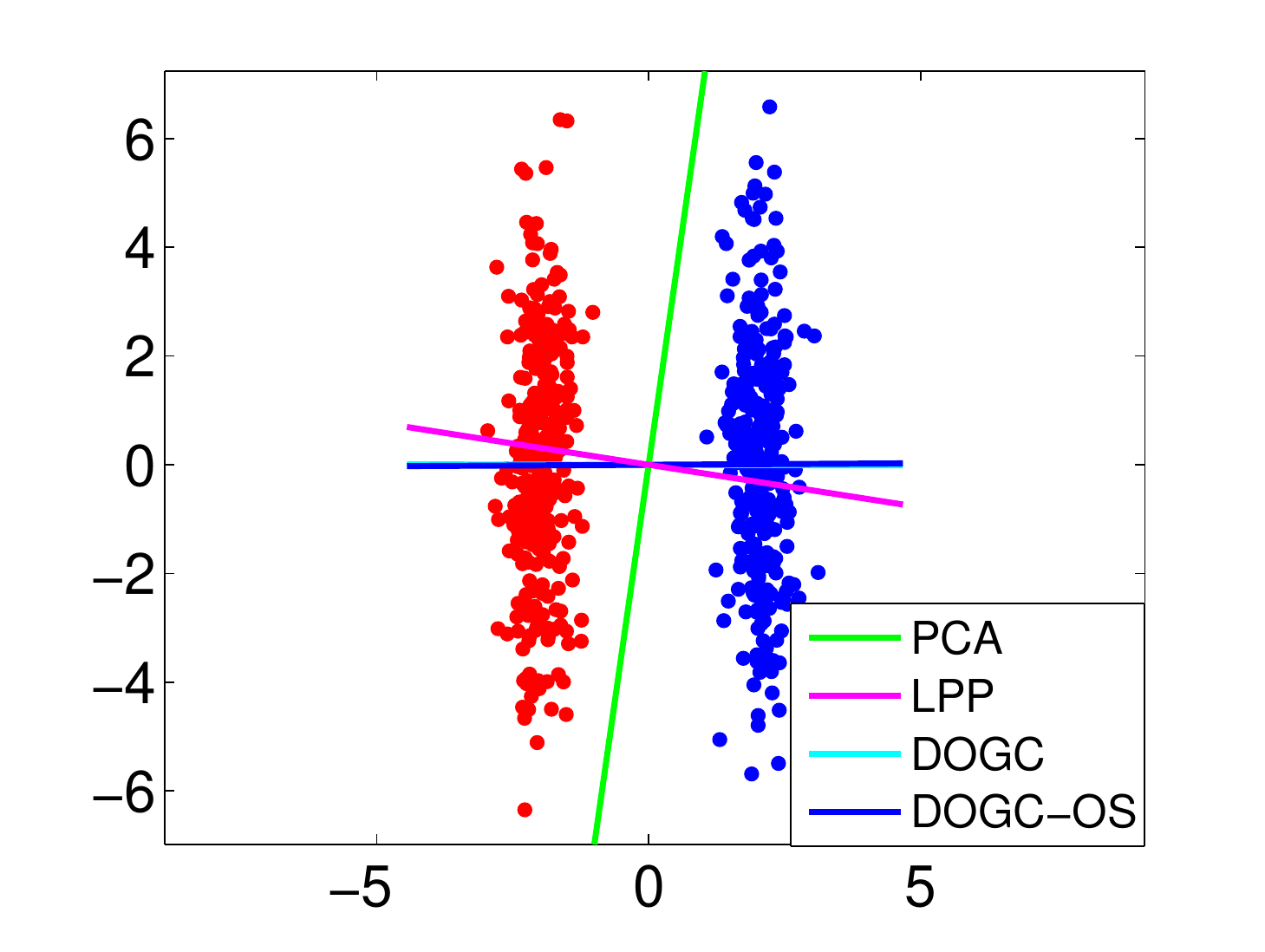}}
\hspace{-0.3in}
\subfigure[Clusters are fairly close]{\includegraphics[width=60mm]{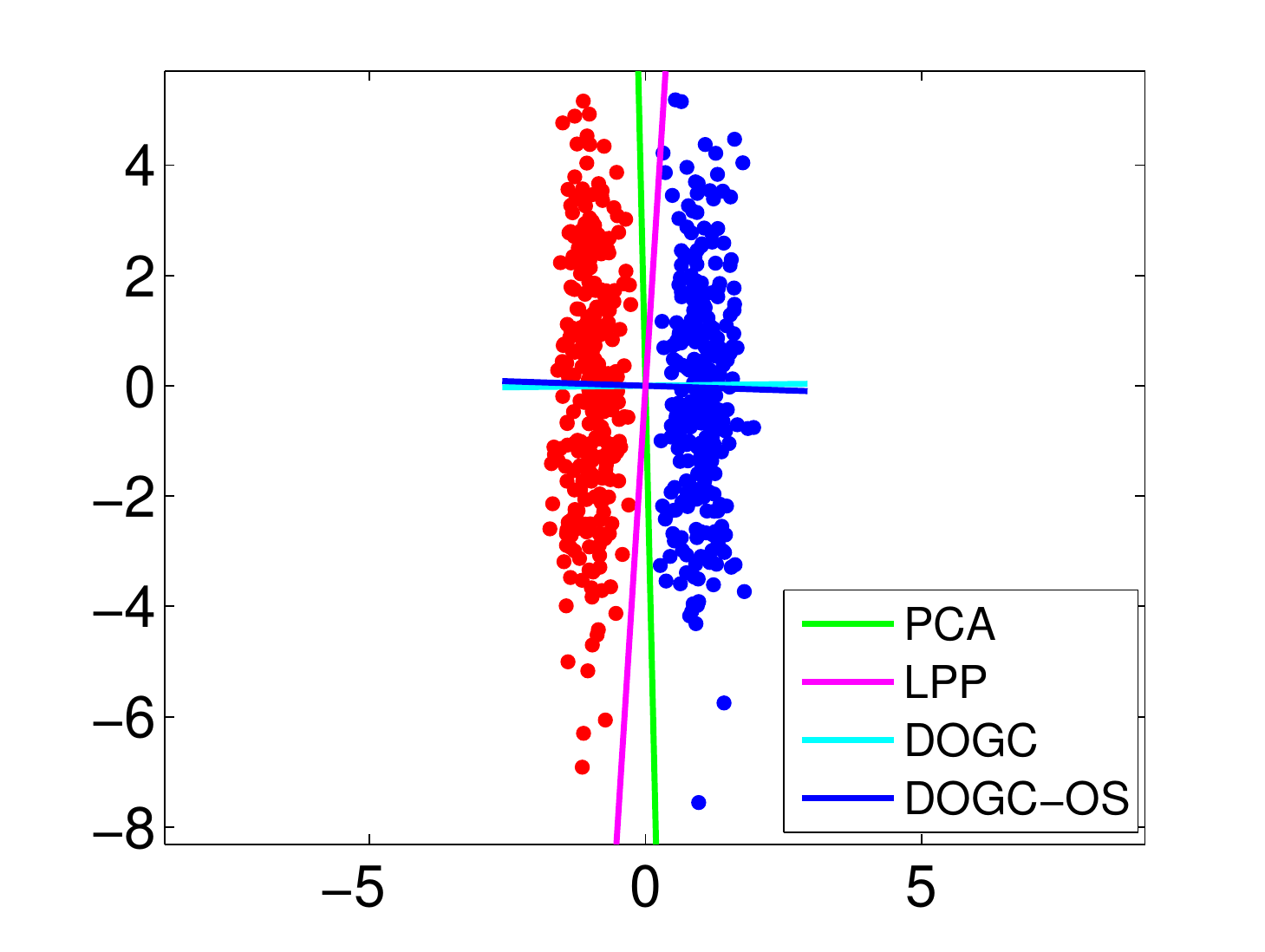}}
\hspace{-0.3in}
\caption{Clustering results on the two-Gaussian synthetic data}
\label{fig:3}
\end{figure*}
\subsection{Effects of Out-of-Sample Extension}
Out-of-sample extension is designed in our approach to predict the unseen data and improve the clustering accuracy. In this subsection, we conduct experiment to evaluate the effects of out-of-sample extension. Specifically, 5 UCI datasets are used to demonstrate the capability of the proposed method on clustering the unseen data. 6 representative clustering methods: k-means (KM)\cite{DBLP:journals/J.B.MacQueen}, spectral clustering (SC)\cite{DBLP:conf/nips/NgJW01}, spectral embedding clustering (SEC)\cite{DBLP:journals/tnn/NieZTXZ11}, discriminative k-means (DKM)\cite{DBLP:conf/nips/YeZW07}, local learning (LL)\cite{DBLP:conf/nips/WuS06}, clustering with local and global regularization (CLGR)\cite{DBLP:journals/tkde/WangZL09} are used for performance comparison. On each dataset, we randomly choose $50\%$ of data samples for training and the rest for testing. In DOGC-OS, the training data is used to train projection matrix $\textbf{P}$ with which the discrete cluster labels of testing data are obtained by projection process. For other methods, we import two parts of data together into their models and report their clustering. As shown in Table \ref{table:4} and Table \ref{table:5.5}, DOGC-OS achieves higher ACC than other methods on both training data and testing data. Furthermore, to evaluate the effects of out-of-sample extension on improving the clustering performance, we compare DOGC-OS with DOGC that removes the part of out-of-sample extension. The results are shown in Table \ref{table:9} and Table \ref{table:10}. From them, we can clearly observe that DOGC-OS can achieve superior performance in most datasets.
\begin{table*}
\caption{ACC on training part of five datasets}
\label{table:4}
\centering
\begin{tabular}{|c|c|c|c|c|c|c|c|}
\hline
\multirow{1}{*}{{Methods}} & \multicolumn{1}{c|}{\textbf{KM}} & \multicolumn{1}{c|}{\textbf{DKM}} & \multicolumn{1}{c|}{\textbf{SC}}&\multicolumn{1}{c|}{\textbf{LL}}&\multicolumn{1}{c|}{\textbf{CLGR}}&
\multicolumn{1}{c|}{\textbf{SEC}}&\multicolumn{1}{c|}{\textbf{DOGC-OS}}\\
\hline
    Solar & 0.4568 & 0.5679 & 0.4444 & 0.4074 & 0.5370 & 0.5802  &\textbf{0.6147}\\ \hline
    Vehicle & 0.4539& 0.4539 & 0.4941 & 0.5532 & 0.5177 & 0.4681  & \textbf{0.5626}\\ \hline
    Vote & 0.7844 & 0.8624 & 0.8394 & 0.8716 & 0.8440 & 0.7936 &  \textbf{0.9309}\\ \hline
    Ecoli & 0.6488 & 0.6667 & 0.3631 & 0.4286 & 0.5179 & 0.6429 &\textbf{0.8333}\\ \hline
    Wine & 0.7079 & 0.7079 & 0.7079 & 0.5843 & 0.7079 & 0.7416 &  \textbf{0.9326}\\ \hline
\end{tabular}
\vspace{-2mm}
\end{table*}
\begin{table*}
\caption{ACC on Testing part of five datasets}
\label{table:5.5}
\centering
\begin{tabular}{|c|c|c|c|c|c|c|c|}
\hline
\multirow{1}{*}{{Methods}} & \multicolumn{1}{c|}{\textbf{KM}} & \multicolumn{1}{c|}{\textbf{DKM}} & \multicolumn{1}{c|}{\textbf{SC}}&\multicolumn{1}{c|}{\textbf{LL}}&\multicolumn{1}{c|}{\textbf{CLGR}}&
\multicolumn{1}{c|}{\textbf{SEC}}&\multicolumn{1}{c|}{\textbf{DOGC-OS}}\\
\hline
    Solar & 0.5155 & 0.5466 & 0.3043 & 0.5217 & 0.6025 & 0.5776 & \textbf{0.6089}\\ \hline
    Vehicle & 0.4468 & 0.4586 & 0.3877 & 0.5272 & 0.4894 & 0.4539 & \textbf{0.5302}\\ \hline
    Vote & 0.8065 & 0.8433 & 0.7926 & 0.8986 & 0.8525 & 0.8111 & \textbf{0.9401}\\ \hline
    Ecoli & 0.7321 & 0.7262 & 0.6548 & 0.7202 & 0.6726 & 0.7381 & \textbf{0.7738}\\ \hline
    Wine & 0.7079 & 0.7079 & 0.6292 & 0.5393 & 0.7753 & 0.6966 & \textbf{0.9326}\\ \hline
\end{tabular}
\vspace{-2mm}
\end{table*}
\begin{figure}
\includegraphics[width=43mm]{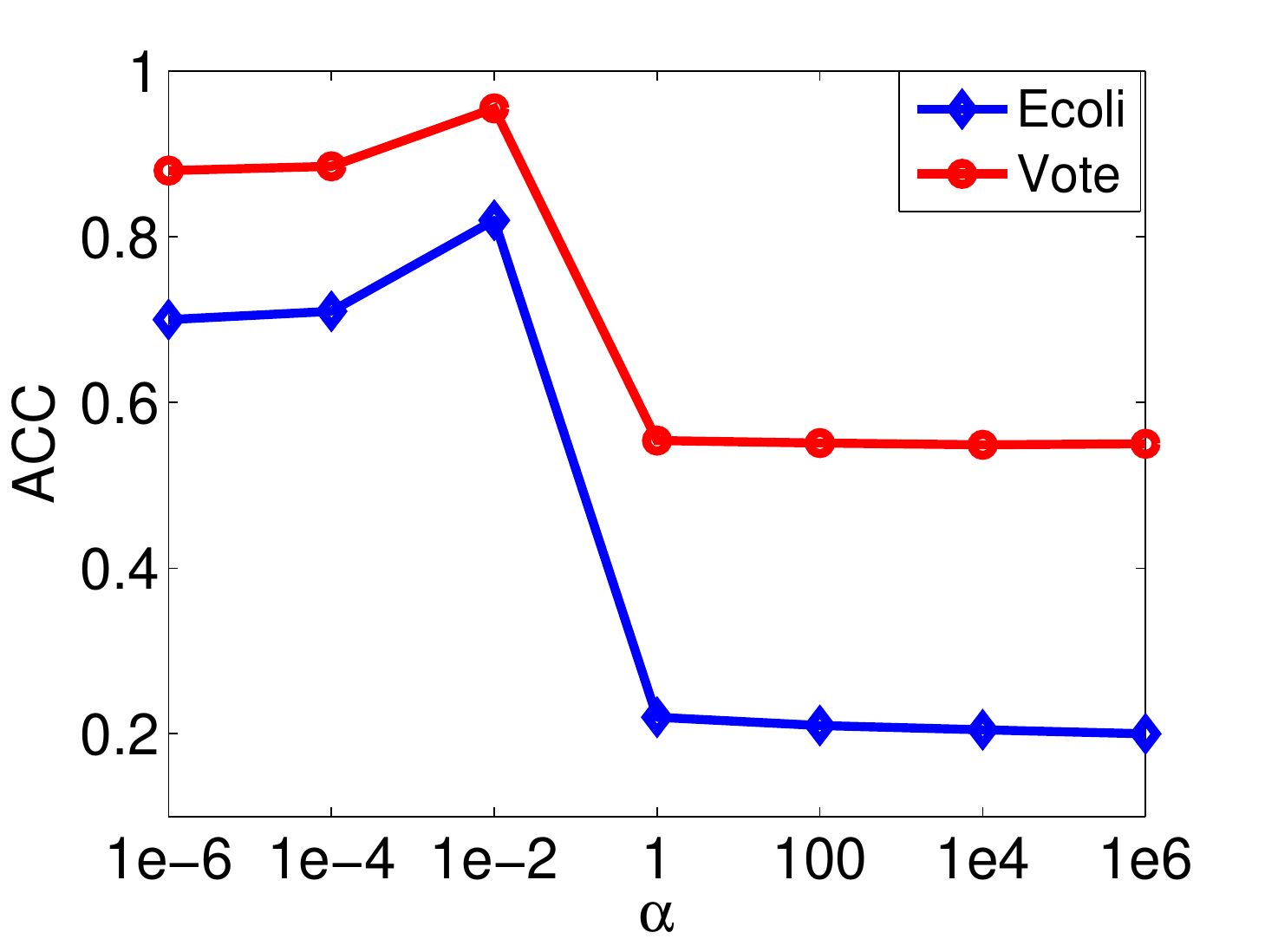}
\hspace{-0.1in}
\includegraphics[width=43mm]{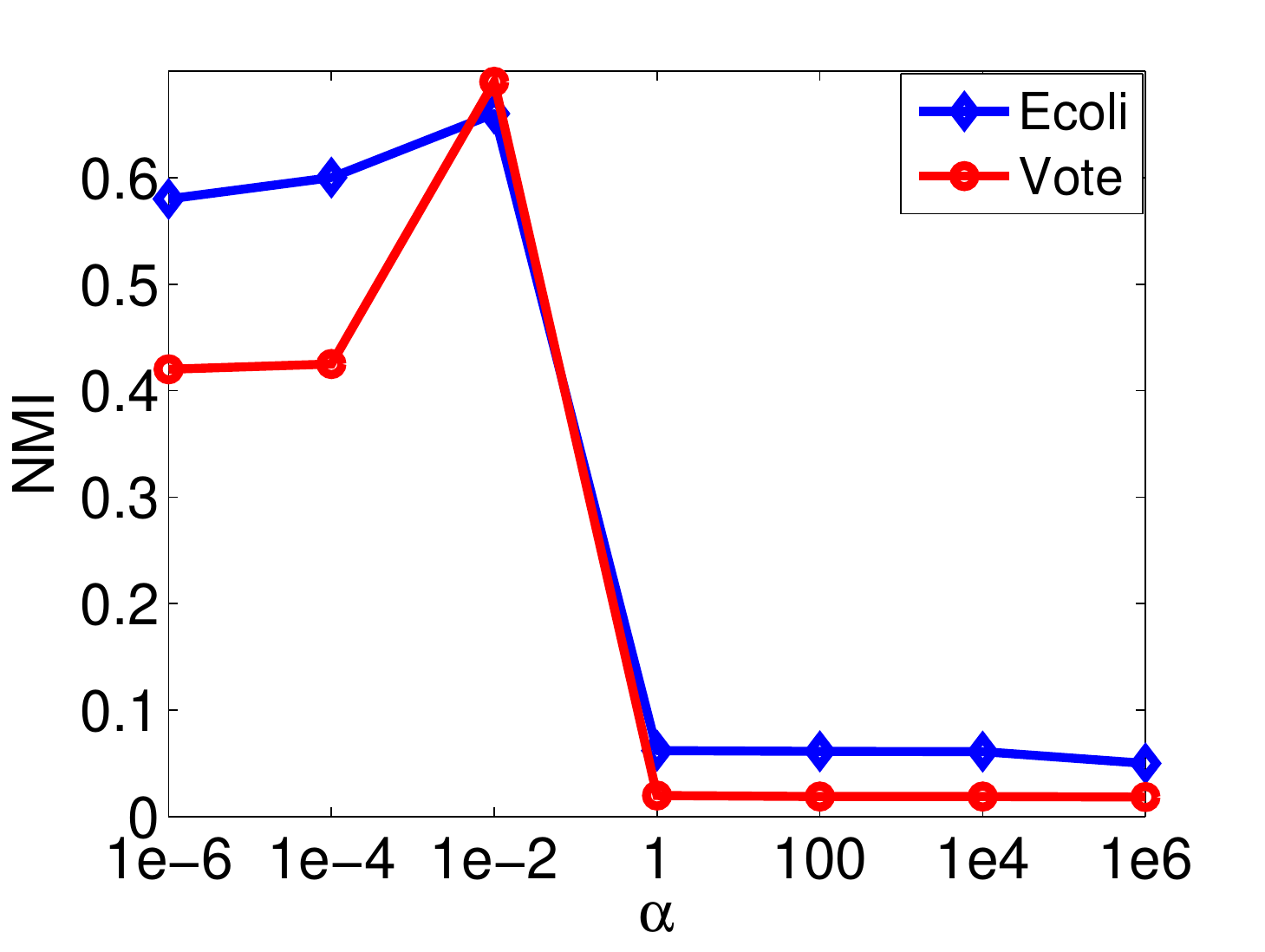}
\caption{ACC and NMI variations with the parameter $\emph{$\alpha$}$ in DOGC.}
\label{fig:7}
\end{figure}
\begin{figure*}
\centering
\subfigure[ACC (Vote)-Sensitivity of $p$]{\label{fig:(1)}\includegraphics[width=42mm]{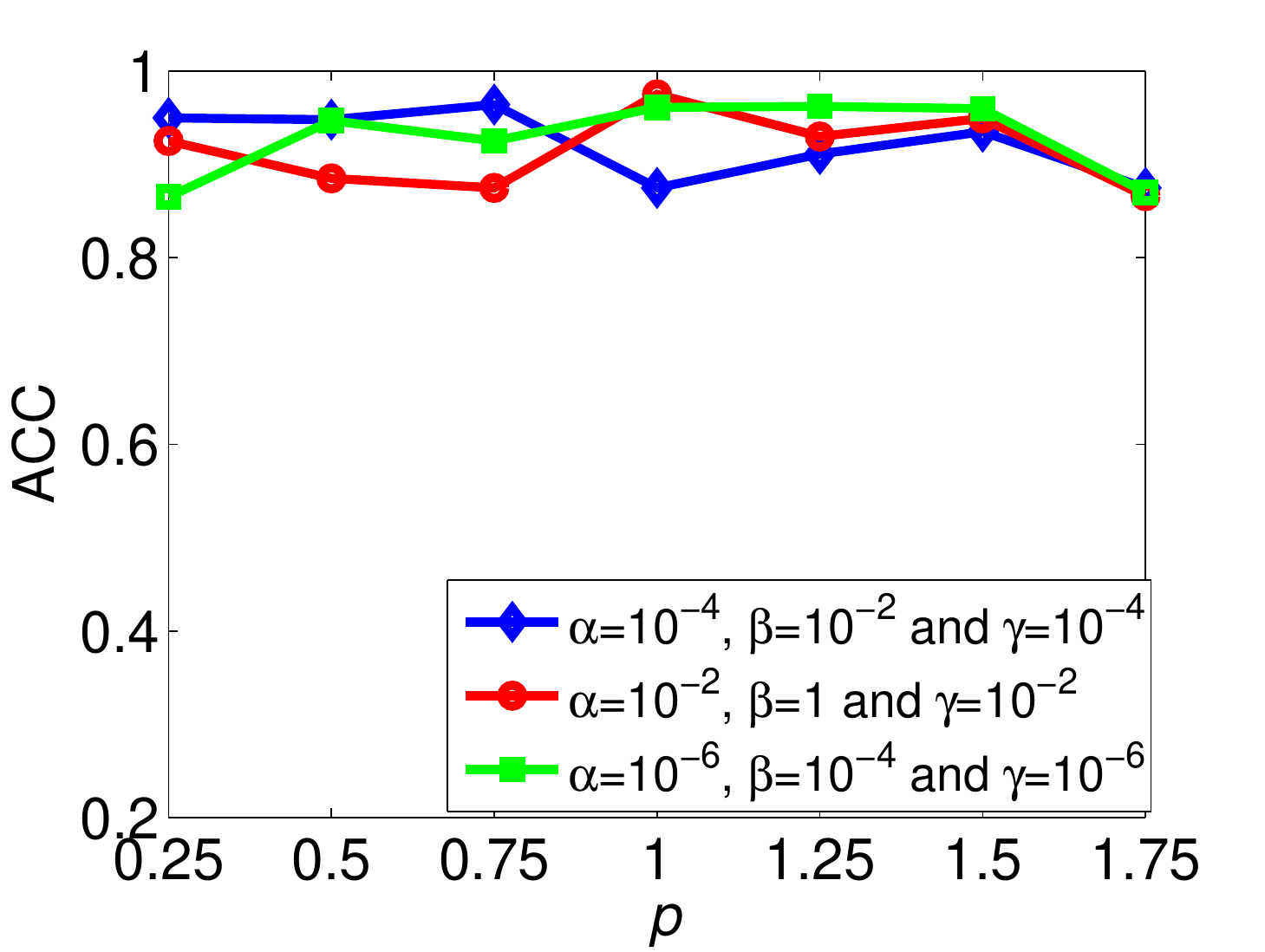}}
\hspace{-0.15in}
\subfigure[NMI (Vote)-Sensitivity of $p$]{\label{fig:(2)}\includegraphics[width=42mm]{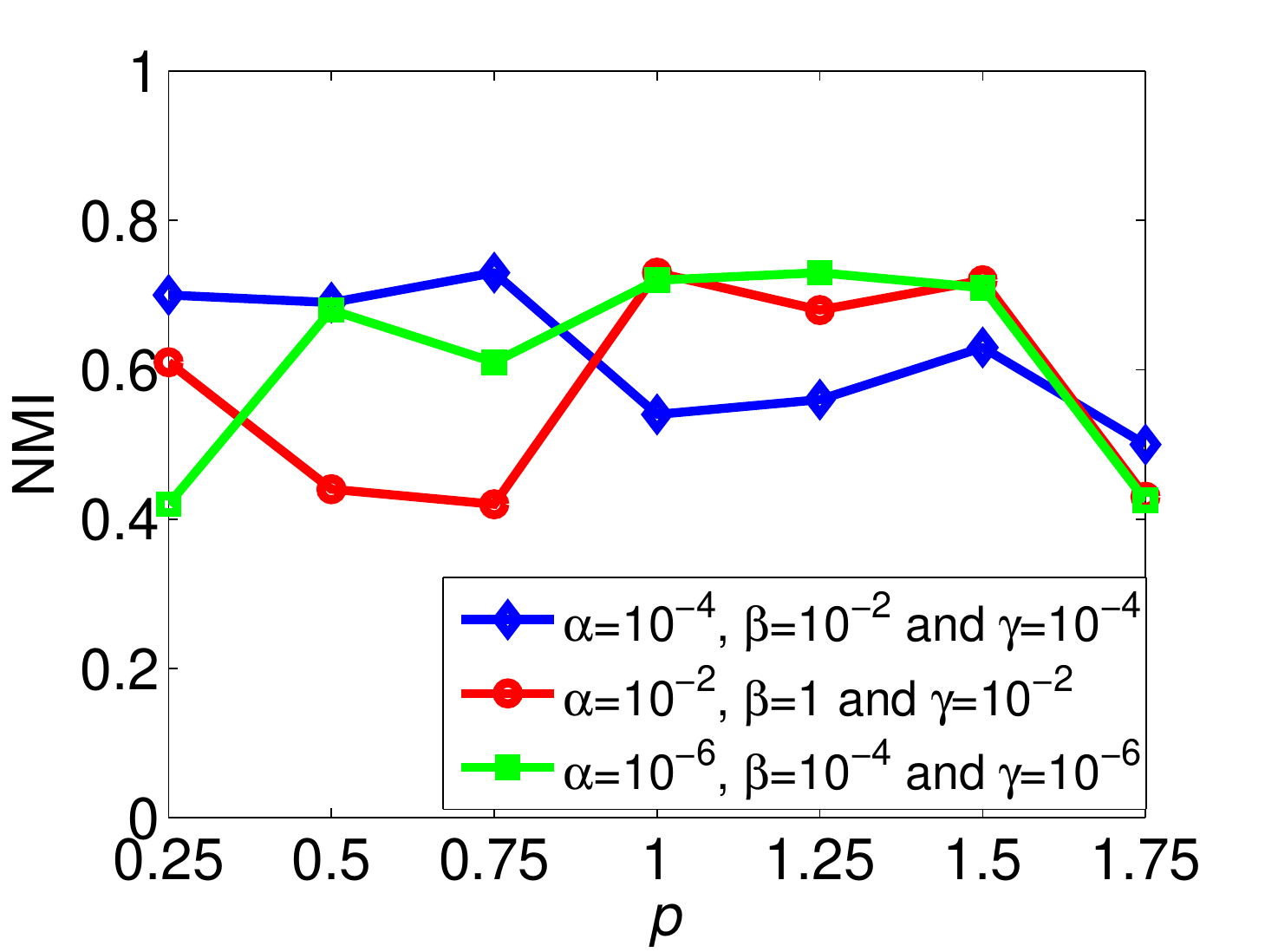}}
\hspace{-0.15in}
\subfigure[ACC (Ecoli)-Sensitivity of $p$]{\label{fig:(3)}\includegraphics[width=42mm]{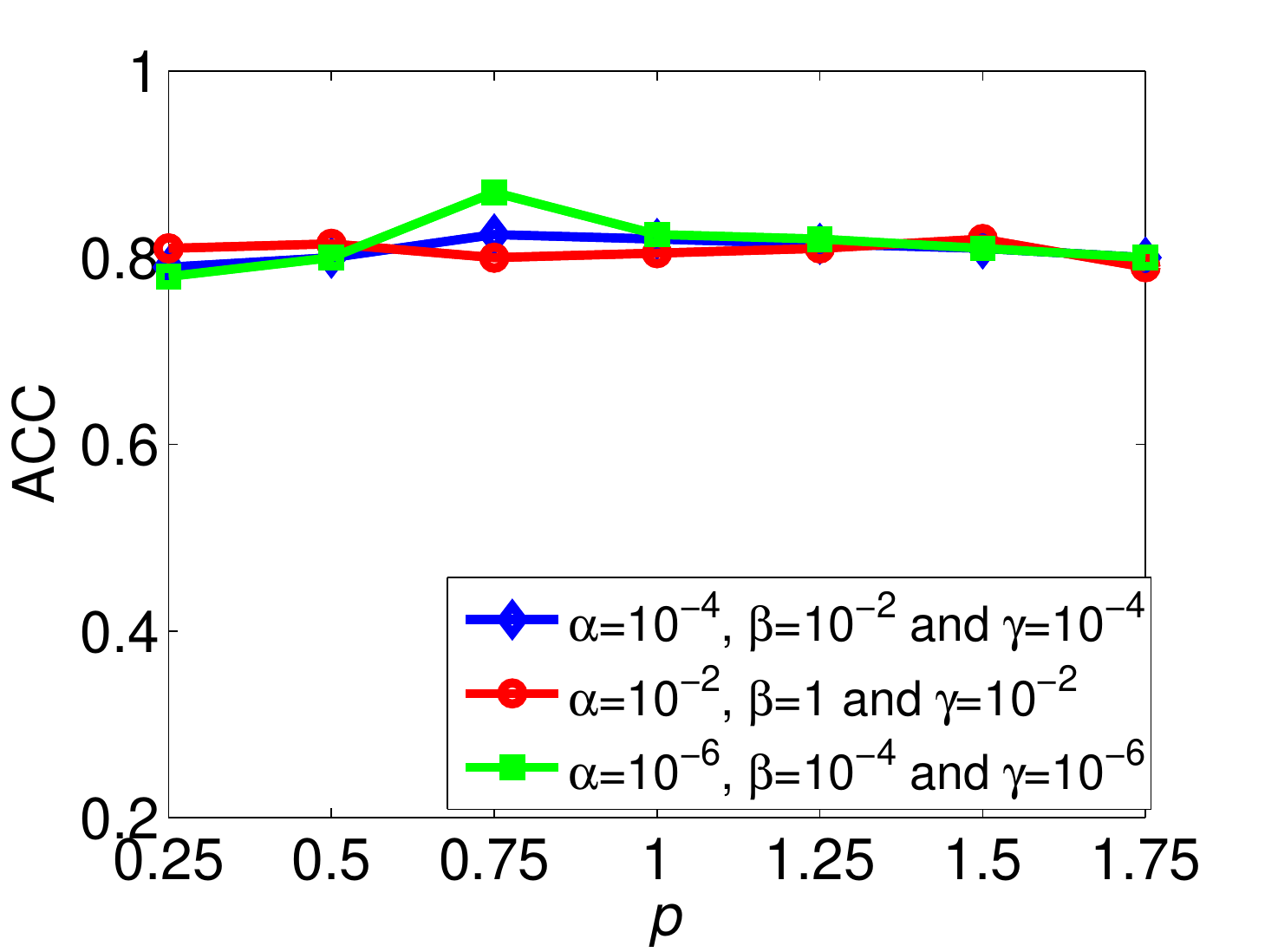}}
\hspace{-0.15in}
\subfigure[NMI (Ecoli)-Sensitivity of $p$]{\label{fig:(4)}\includegraphics[width=42mm]{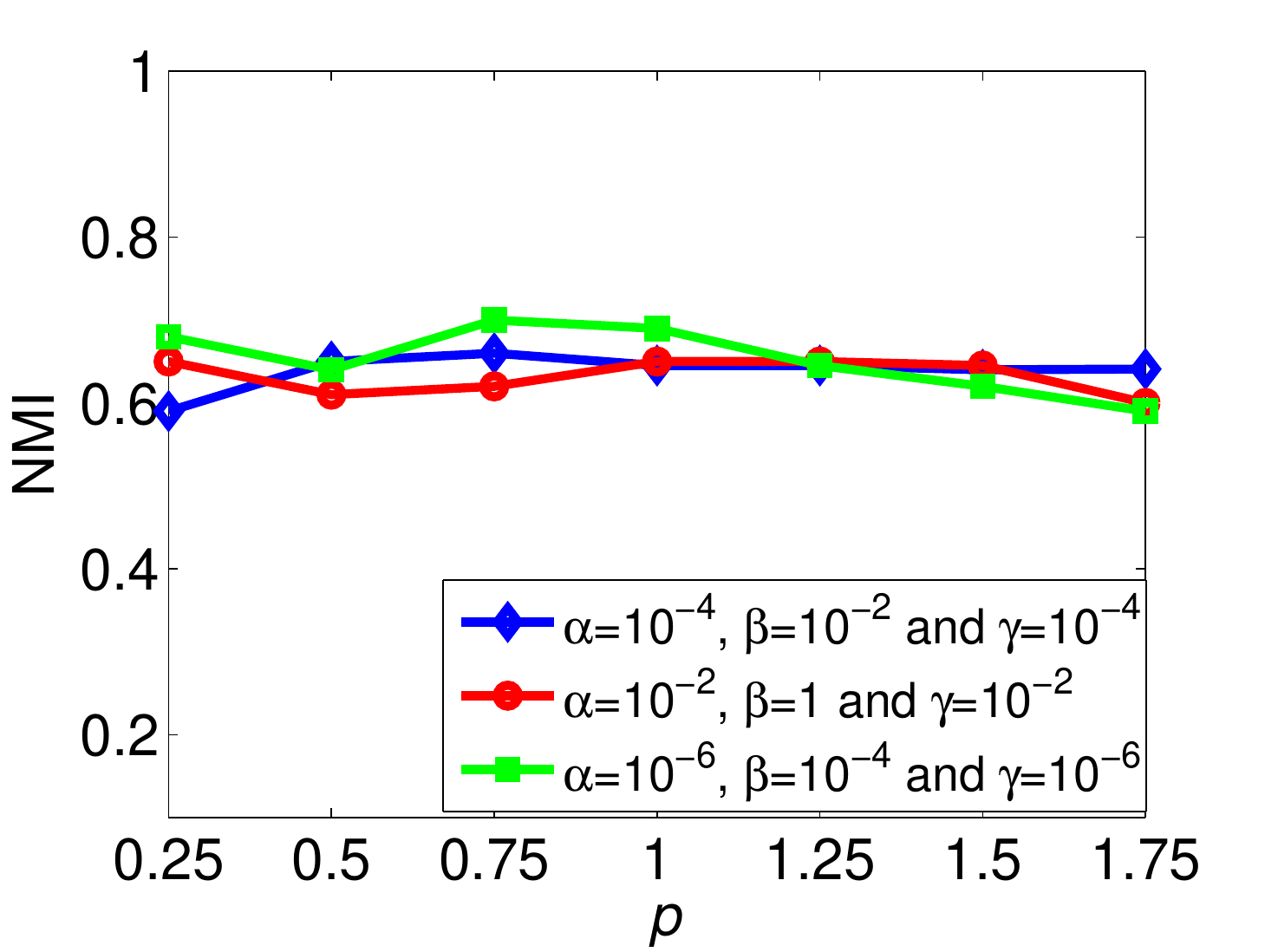}}
\caption{Parameter variations with parameter $p$ in DOGC-OS}
\label{fig:6}
\end{figure*}
\begin{figure*}
\vspace{-4mm}
\centering
\subfigure[$\alpha$=0.0001]{\includegraphics[width=48mm]{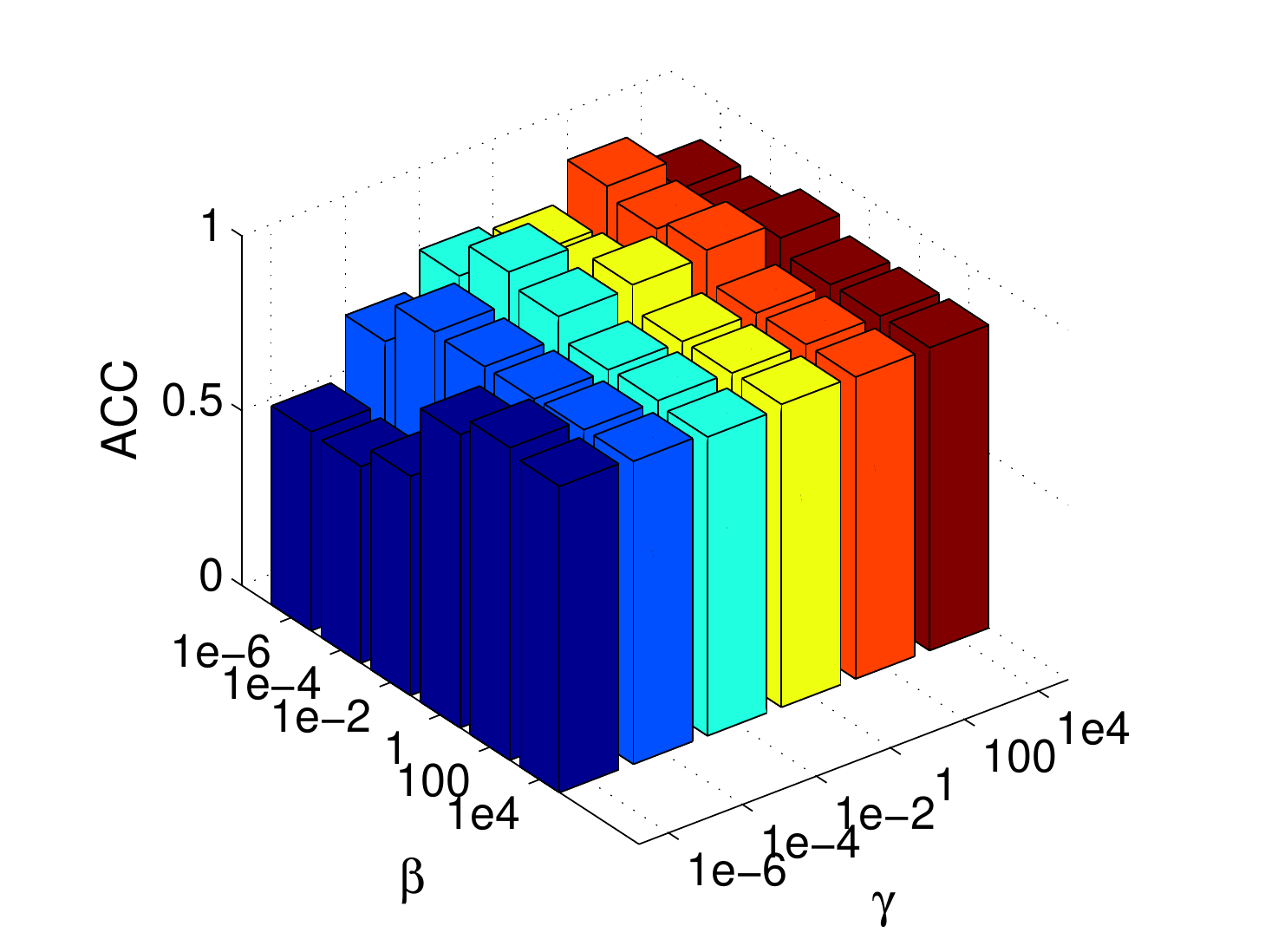}}
\hspace{-0.2in}
\subfigure[$\gamma$=0.0001]{\includegraphics[width=48mm]{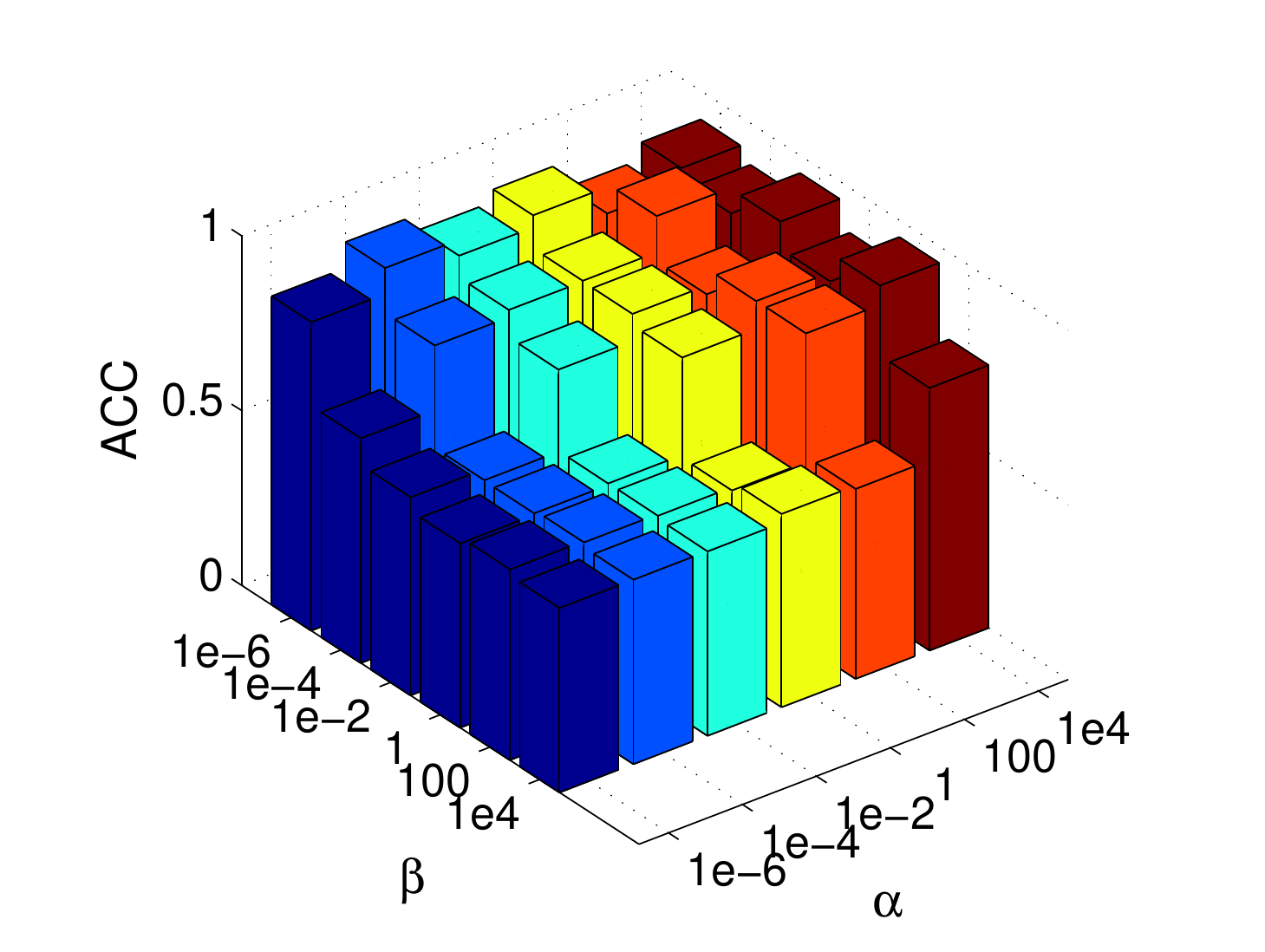}}
\hspace{-0.2in}
\subfigure[$\beta$=0.01]{\includegraphics[width=48mm]{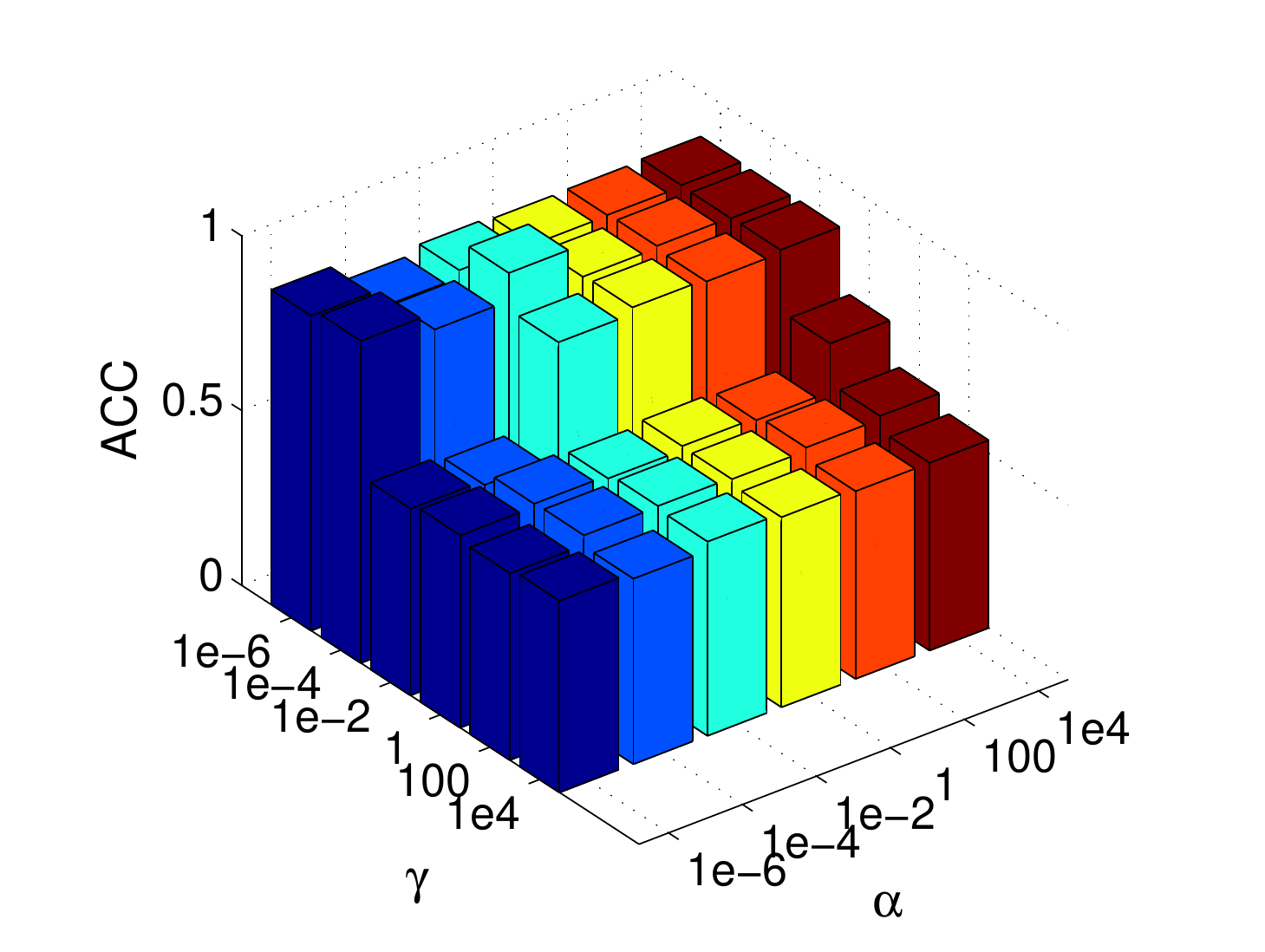}}
\hspace{-0.2in}
\caption{Performance variations with the parameters $\alpha$, $\beta$ and $\gamma$ on Vote.}
\label{fig:4}
\end{figure*}
\begin{figure*}
\vspace{-3.5mm}
\centering
\subfigure[$\alpha$=0.000001]{\includegraphics[width=48mm]{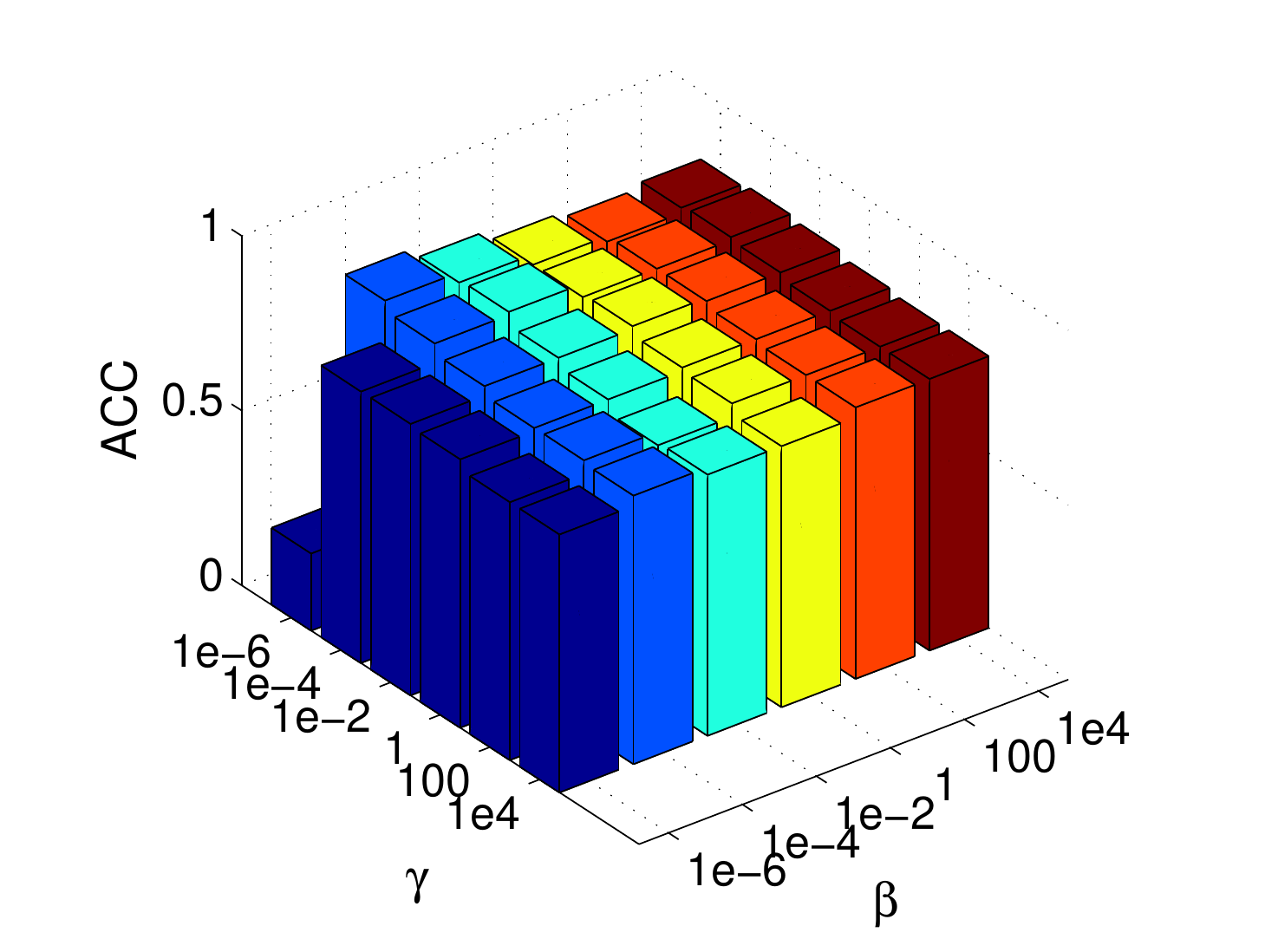}}
\hspace{-0.2in}
\subfigure[$\gamma$=0.000001]{\includegraphics[width=48mm]{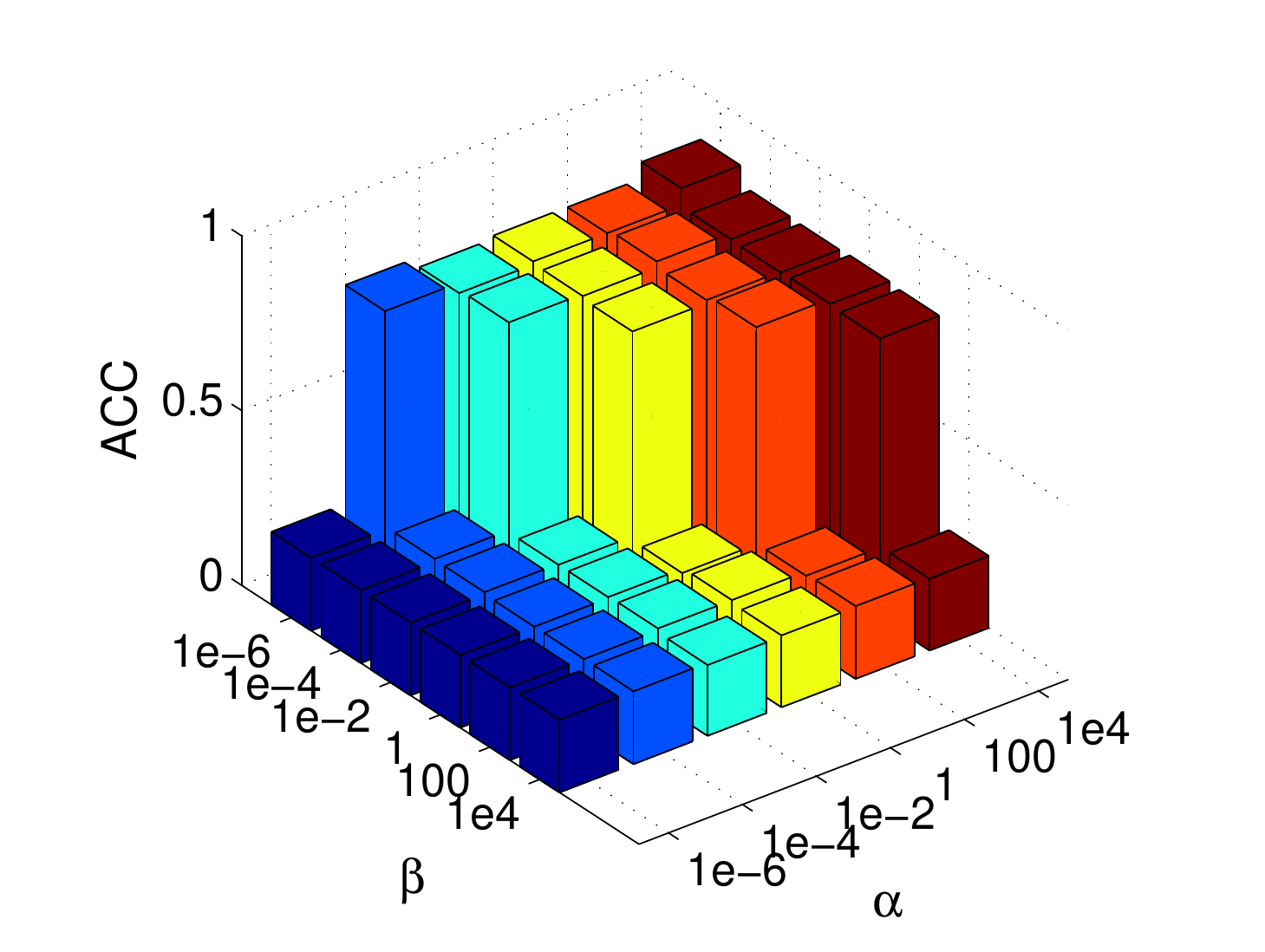}}
\hspace{-0.2in}
\subfigure[$\beta$=0.001]{\includegraphics[width=48mm]{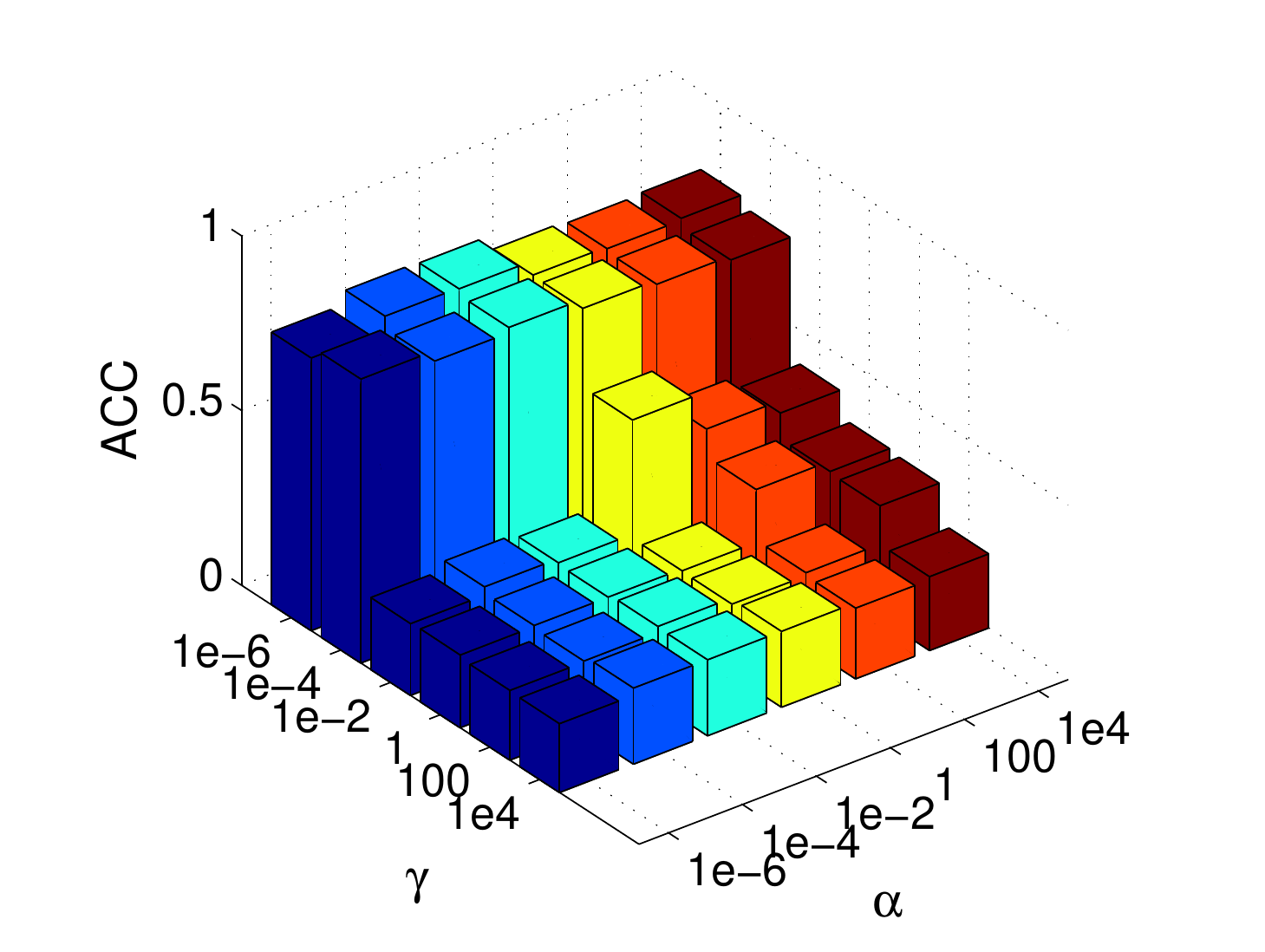}}
\caption{Performance variations with the parameters $\alpha$, $\beta$ and $\gamma$ on Ecoli.}
\label{fig:5}
\end{figure*}

\subsection{Parameter Sensitivity Experiment}
There are three parameters: $\alpha$, $\beta$ and $\gamma$ in DOGC-OS, and one parameter $\alpha$ in DOGC. In this subsection, we perform experiments on Vote and Ecoli to evaluate the parameter sensitivity of the proposed methods, and investigate how they perform on different parameter settings. In experiment, we tune the parameters $\alpha$, $\beta$ and $\gamma$ in the range of $\big\{10^{-6}, 10^{-4}, 10^{-2}, 1, 10^{2}, 10^{4} \big\}$, and $p$ from 0.25 to 1.75. In DOGC, we observe the variations of ACC and NMI with $\alpha$ from $10^{-6}$ to $10^{4}$ (as shown in Figure \ref{fig:7}). In DOGC-OS, we first select three groups of parameters with fixed $\alpha$, $\beta$ and $\gamma$ (as shown in Figure \ref{fig:6}). From it, we find that $p$=1.25 is optimal for clustering. Then, we fix $p$=1.25, and evaluate the performance variations with remaining $\alpha$, $\beta$, $\gamma$. Specifically, we fix two parameters and observe the variations of ACC with the other one. Figure \ref{fig:4} and \ref{fig:5} illustrate the main experimental results. The analysis of the parameters effects are as follows:
\begin{figure*}
\centering
\subfigure[DOGC (Ecoli)]{\includegraphics[width=45mm]{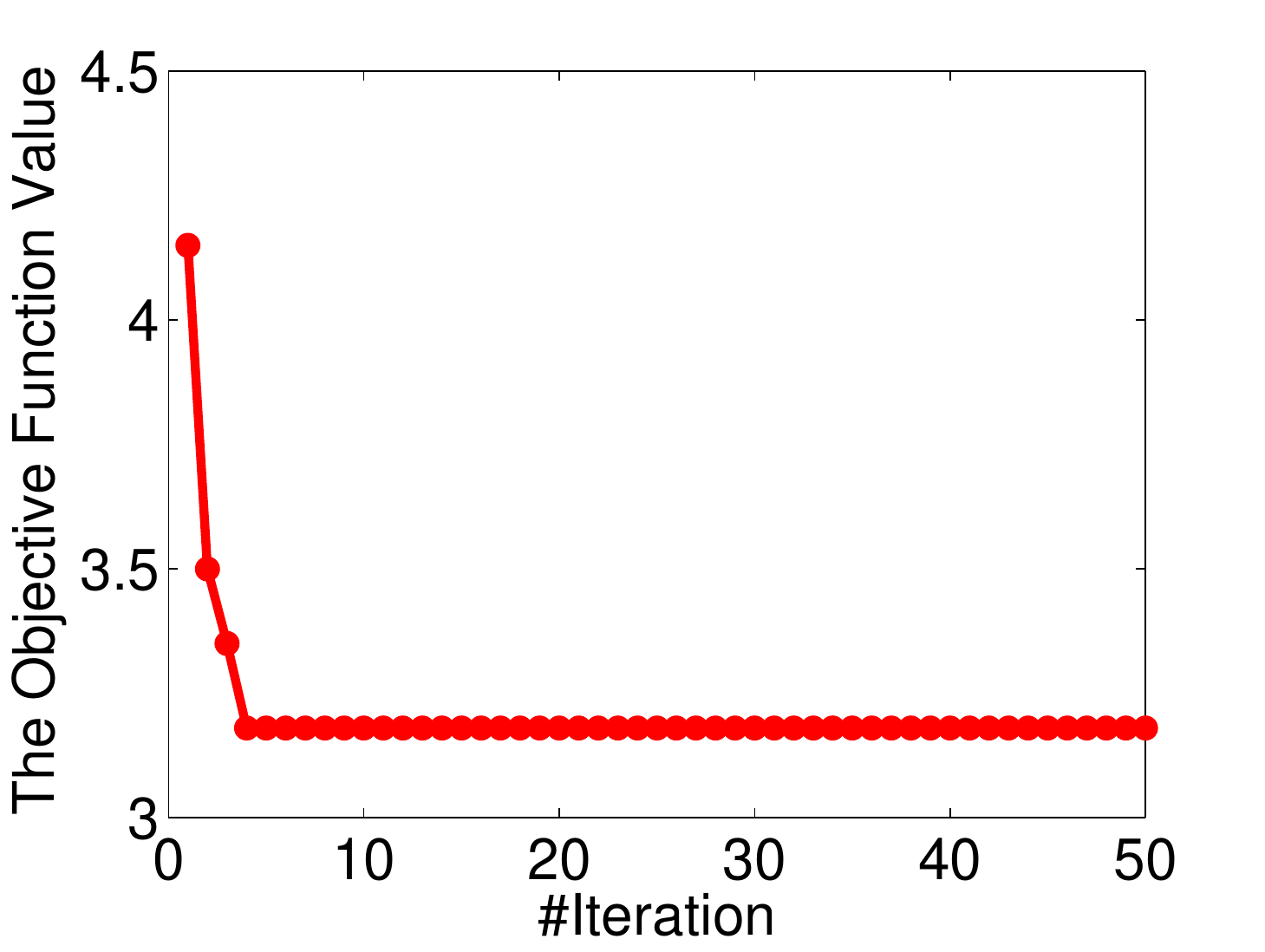}}
\hspace{-0.2in}
\subfigure[DOGC (Vote)]{\includegraphics[width=45mm]{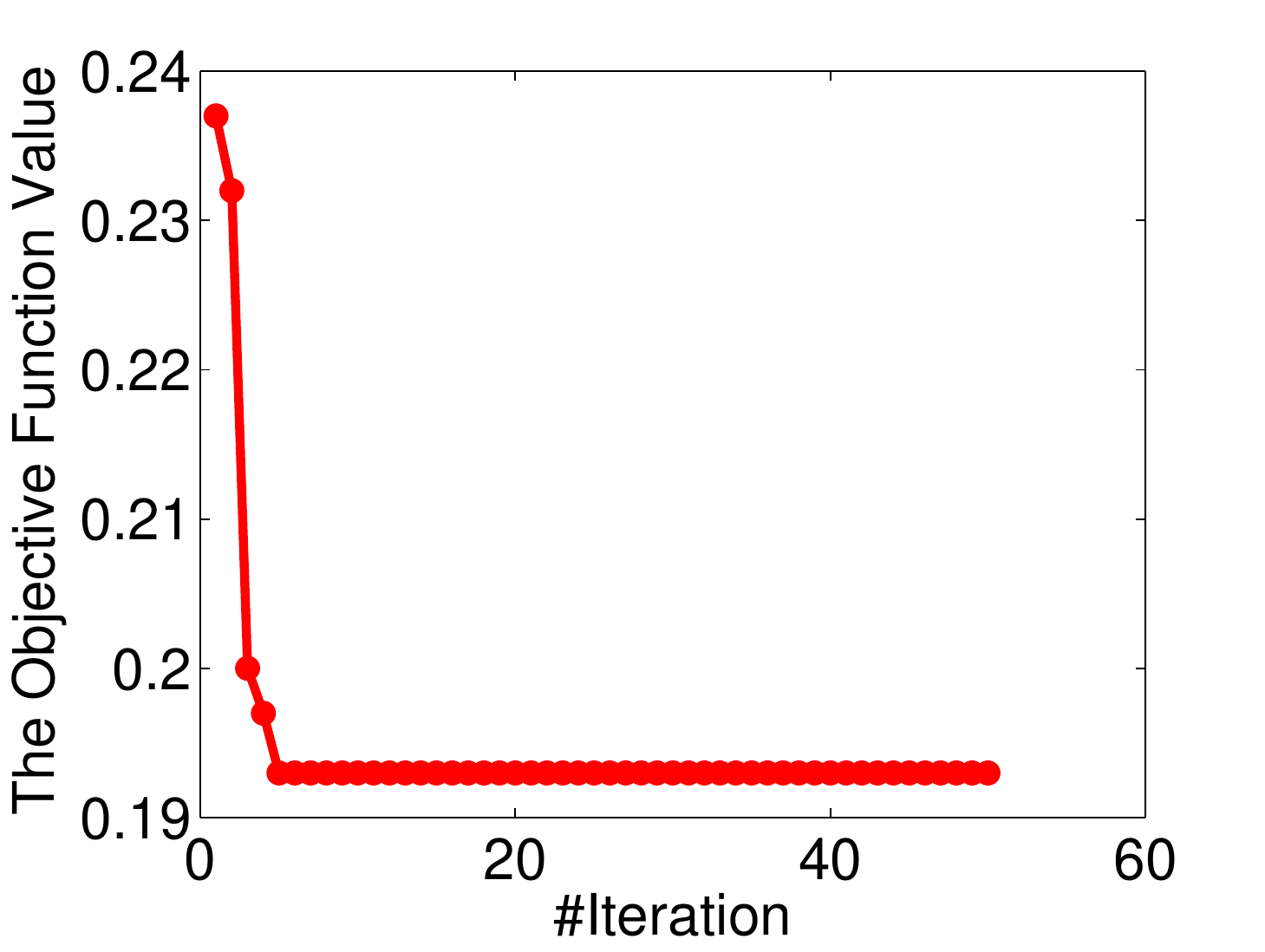}}
\hspace{-0.2in}
\subfigure[DOGC-OS (Ecoli)]{\includegraphics[width=45mm]{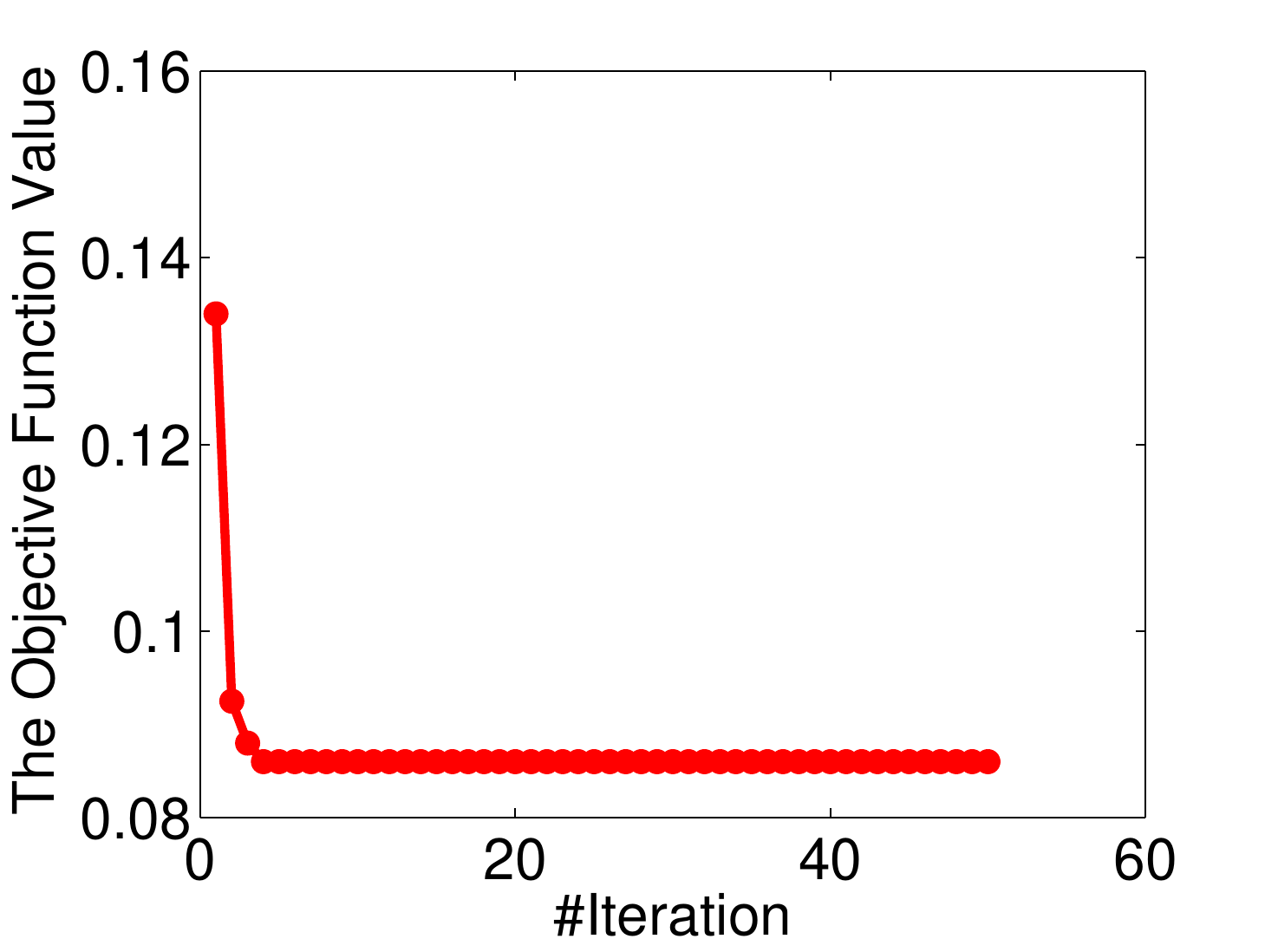}}
\hspace{-0.2in}
\subfigure[DOGC-OS (Vote)]{\includegraphics[width=45mm]{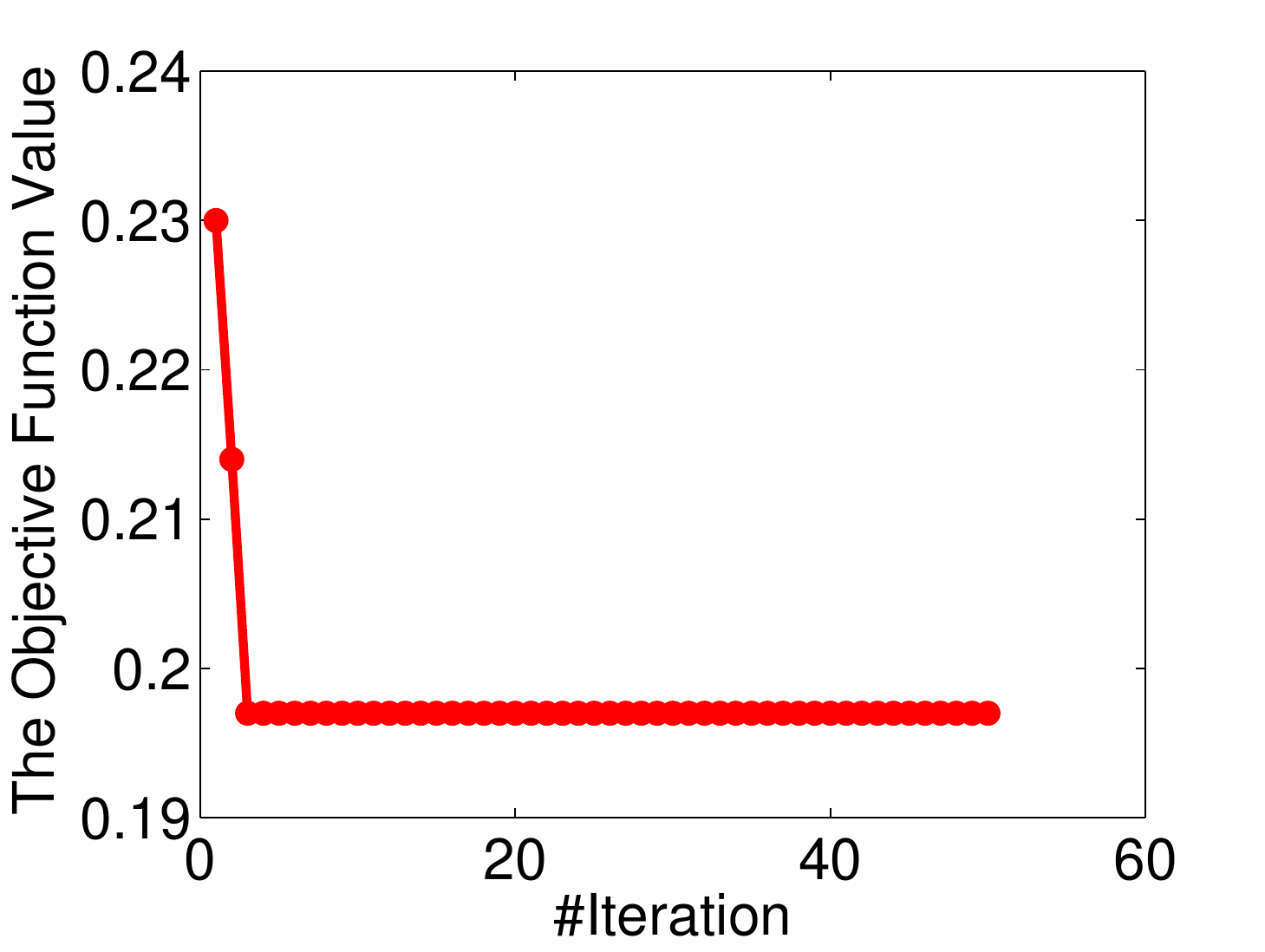}}
\caption{The variations of objective function value with the number of iterations on two datasets.}
\label{fig8888}
\end{figure*}
\begin{itemize}
\item \textbf{Joint effects of $\alpha$ and $\beta$.}
$\alpha$ contributes to discrete label learning. In DOGC and DOGC-OS, $\alpha$ plays a crucial role in clustering performance. In experiments, we find that our methods can achieve satisfactory performance with $\alpha$ in the range of $\big\{ 10^{-6},10^{-4},10^{-2} \big\}$. In DOGC-OS, when $\alpha$ is small, we observe a decreasing trend of performance as $\beta$ increases. Once $\alpha$ is larger than $\beta$, our method works better instead.
\item \textbf{Effects of $\gamma$.}
$\gamma$ controls the prediction residual error $\parallel \textbf{Y}-\textbf{X}^\top \textbf{P}\parallel_{2,p}$. In DOGC-OS, as $\gamma$ increases from $10^{-2}$ to $10^{-1}$, it performs gradually better. When $\gamma$ keeps going up, we observe a decreasing trend instead. If $\gamma$ is small, the regularization term will become less significant and over-fitting problem may be brought. On the contrary, when we use large $\gamma$, the prediction residual error will not be well controlled. Under this circumstance, our approach DOGC-OS will produce sub-optimal prediction function and discrete cluster labels.
\item \textbf{Effects of $p$.}
In DOGC-OS, when $\alpha$ is optimal, $p$ will have less influence on the clustering performance. The main reason is that the influence of $\alpha$ covers $p$. In contrast, if $\alpha$ is not optimal, the influence of $p$ on ACC gradually becomes important. Under such circumstance, when $p$ is in the range of 1 to 1.5, it can help to improve ACC.
\end{itemize}

\subsection{Convergence Experiment}
We have theoretically proved that the proposed iterative optimization can obtain a converged solution. In this subsection, we empirically evaluate the convergence of the proposed algorithms. We conduct experiments on Vote and Ecoli. Similar results can be obtained on other datasets. Figure \ref{fig8888} records the variations of objective function value of Eq.(\ref{eq:7}) and Eq.(\ref{eq:26}) respectively with the number of iterations. It clearly shows that our approaches are able to converge rapidly within only a few iterations (less than 10).

\section{Conclusions}
In this paper, we propose a unified discrete optimal graph clustering framework. In our methods, a structured graph is adaptively learned with the guidance of a reasonable rank constraint to support the clustering. Discrete cluster labels are directly learned by learning a proper rotation matrix to avoid the relaxing information loss. Besides, projective subspace learning is performed in our framework to eliminate noise and extract discriminative information from raw features to facilitate clustering. Moreover, our model can support the clustering task on the out-of-sample data by designing a robust prediction module. Experiments on both synthetic and real datasets demonstrate the superior performance of the proposed methods.
\ifCLASSOPTIONcaptionsoff
  \newpage
\fi

\section*{\textbf{Acknowledgment}}
The authors would like to thank the anonymous reviewers for their constructive and helpful suggestions.

\ifCLASSOPTIONcaptionsoff
\fi

\bibliographystyle{IEEEtran}
\bibliography{mybibfile}

\begin{IEEEbiography}[{\includegraphics[width=1.2in,height=1.25in,clip,keepaspectratio]{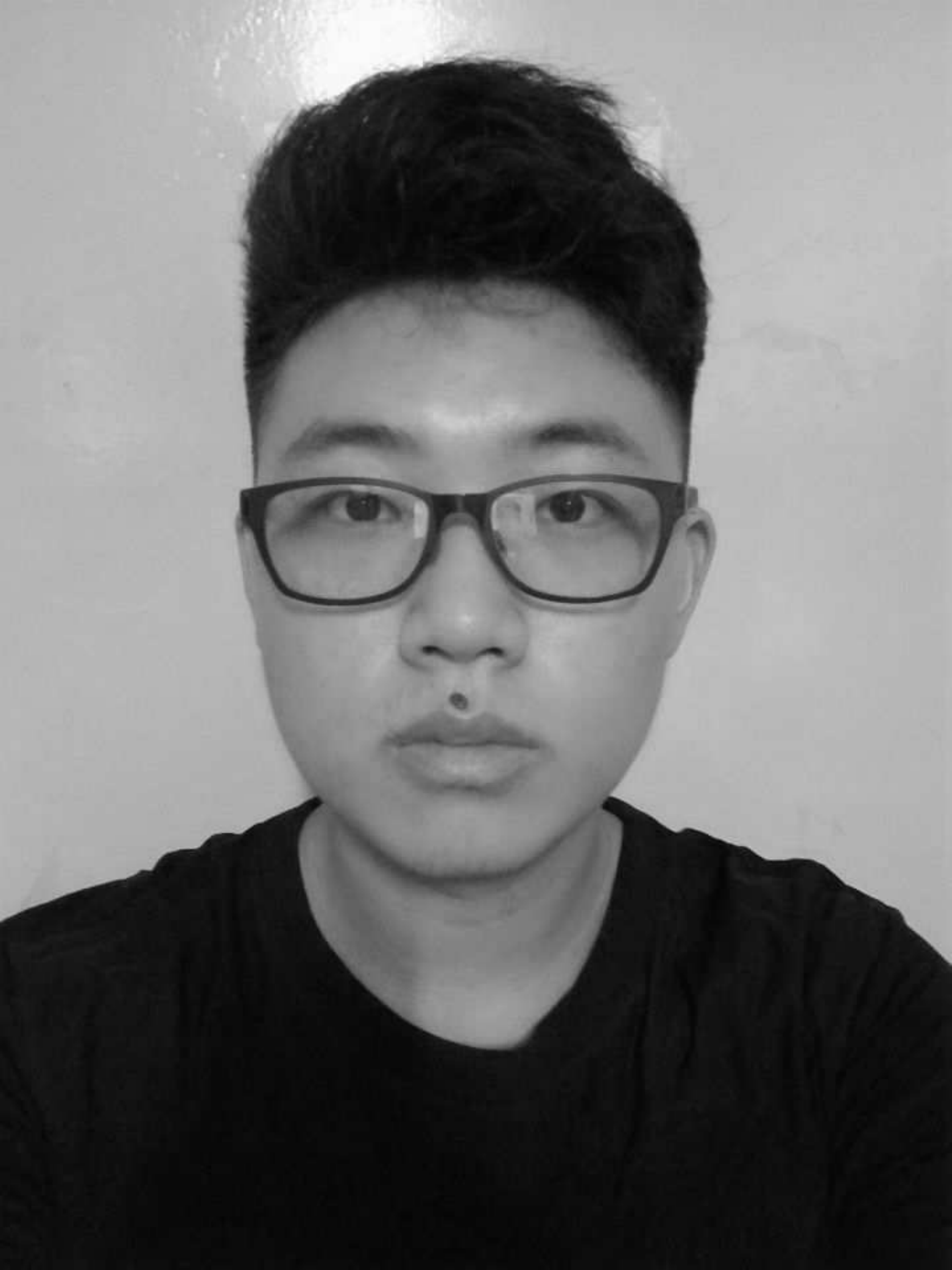}}]{Yudong Han}
is pursuing Bachelor's degree at the School of Information Science and Engineering, Shandong Normal University, China. He is under the supervision of Prof. Lei Zhu (2016-). His research interest is in the area of big data mining.
\end{IEEEbiography}

\begin{IEEEbiography}[{\includegraphics[width=1.2in,height=1.25in,clip,keepaspectratio]{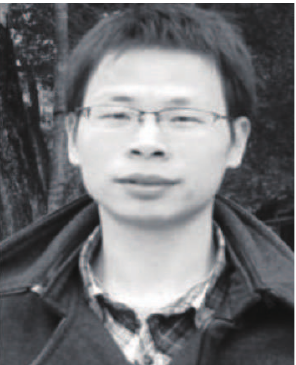}}]{Lei Zhu}
received the B.S. degree (2009) at Wuhan University of Technology, the Ph.D. degree (2015) at Huazhong University of Science and Technology. He is currently a full Professor with
the School of Information Science and Engineering, Shandong Normal University, China.  He was a Research Fellow at the University of Queensland (2016-2017), and at the Singapore Management University (2015-2016). His research interests are in the area of large-scale multimedia content analysis and retrieval.
\end{IEEEbiography}
\begin{IEEEbiography}[{\includegraphics[width=1.2in,height=1.25in,clip,keepaspectratio]{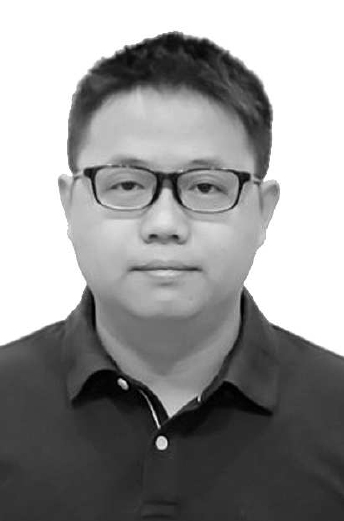}}]{Zhiyong Cheng} received the Ph.D degree in computer science from Singapore Management University, Singapore. He is currently a full professor with Shandong Artificial Intelligence Institute, China. He was a Research Fellow at the National University of Singapore (2016-2018), and  a visiting student with the School of Computer Science, Carnegie Mellon University (2014-2015). His research interests mainly focus on large-scale multimedia content analysis and retrieval. His work has been published in a set of top forums, including ACM SIGIR, MM, WWW, TOIS, IJCAI, TKDE, and TCYB.
\end{IEEEbiography}

\begin{IEEEbiography}[{\includegraphics[width=1.2in,height=1.25in,clip,keepaspectratio]{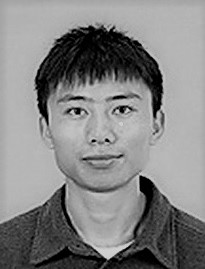}}]{Jingjing Li}
received his MSc and PhD degree in Computer Science from University of Electronic Science and Technology of China in 2013 and 2017, respectively. Now he is a national Postdoctoral Program for Innovative Talents research fellow with the School of Computer Science and Engineering, University of Electronic Science and Technology of China. He has great interest in machine learning, especially transfer learning, subspace learning and recommender systems.
\end{IEEEbiography}

\begin{IEEEbiography}[{\includegraphics[width=1.2in,height=1.25in,clip,keepaspectratio]{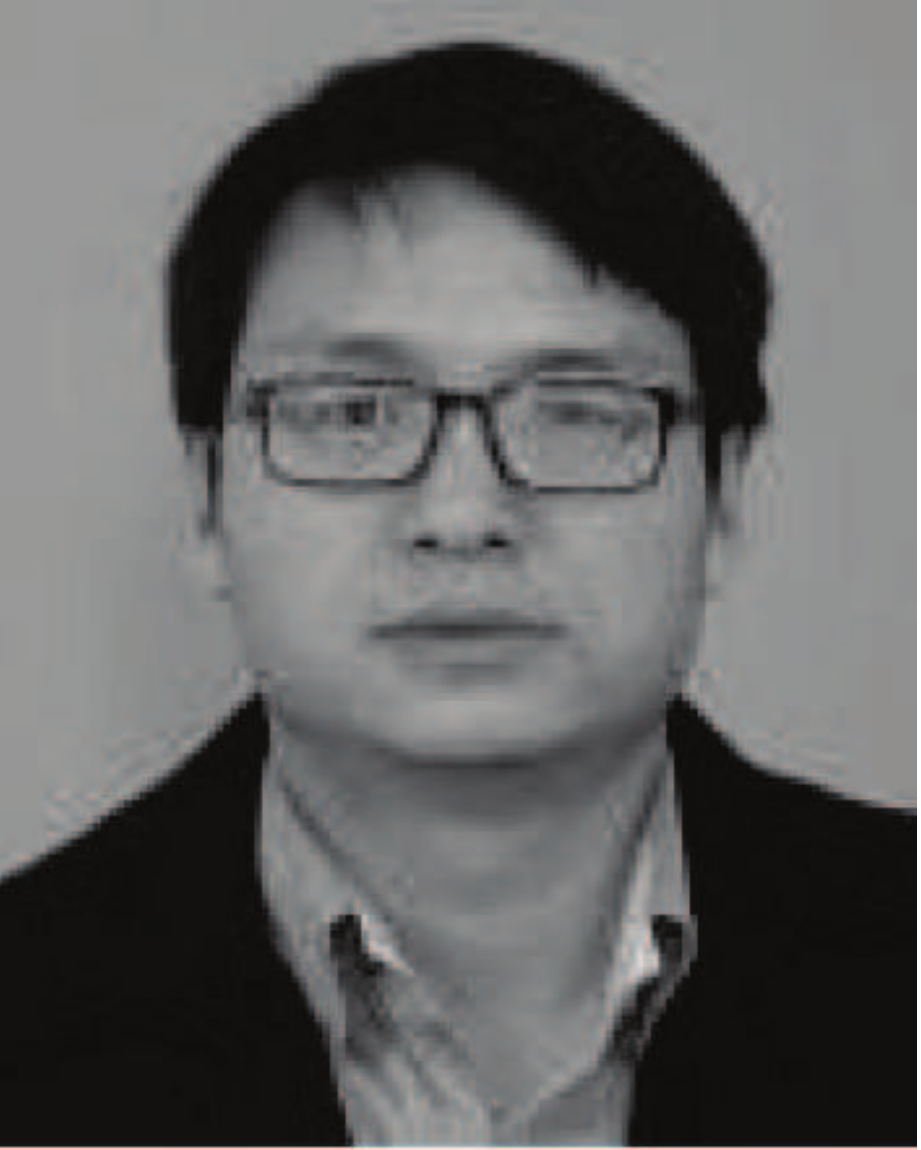}}]{Xiaobai Liu}
received his PhD from the Huazhong University of Science and Technology, China. He is currently an Assistant Professor of Computer Science in the San Diego State University (SDSU), San Diego. His research interests focus on scene parsing with a variety of topics, e.g. joint inference for recognition and reconstruction, commonsense reasonzing, etc. He has published 30+ peer-reviewed articles in top-tier conferences (e.g. ICCV, CVPR etc.) and leading journals (e.g. TPAMI, TIP etc.) He received a number of awards for his academic contribution, including the 2013 outstanding thesis award by CCF (China Computer Federation). He is a member of IEEE.
\end{IEEEbiography}
\vfill
\end{document}